\documentclass[12pt]{article}
\usepackage{graphicx} 
\usepackage{amsmath, amssymb, amsthm,wasysym}
\usepackage{algorithm}
\usepackage{algorithmic}
\usepackage{color}
\usepackage{times}
\usepackage{times}

\usepackage{ifthen}
\newboolean{SHORTVERSION}
\setboolean{SHORTVERSION}{false}

\newtheorem{theorem}{Theorem}   
\newtheorem{lemma}[theorem]{Lemma}
\newtheorem{corollary}[theorem]{Corollary}

\newcommand{\imgdir}{.}

\newcommand{\figscale}[2]{\includegraphics[scale=#2,clip=false]{#1}}

\newcommand{\field}[1]{\mathbb{#1}}
\newcommand{\R}{\field{R}}

\newcommand{\bE}{\field{E}}

\newcommand{\be}{\boldsymbol{e}}

\newcommand{\bu}{\boldsymbol{u}}

\newcommand{\bx}{\boldsymbol{x}}

\newcommand{\by}{\boldsymbol{y}}

\newcommand{\wt}{\widetilde}

\newcommand{\dt}{\displaystyle}
\renewcommand{\ss}{\subseteq}

\newcommand{\norm}[1]{\left\|{#1}\right\|}

\newcommand{\argmin}{\mathop{\rm argmin}}

\newcommand{\bigoheq}{\stackrel{\mathcal{O}}{=}}

\newcommand{\scP}{\mathcal{P}}

\newcommand{\wta}{\textsc{wta}}
\newcommand{\nwwta}{\textsc{nwwta}}
\newcommand{\gpa}{\textsc{gpa}}
\newcommand{\omv}{\textsc{wmv}}
\newcommand{\labprop}{\textsc{labprop}}

\newcommand{\rst}{\textsc{rst}}

\newcommand{\nwrst}{\textsc{nwrst}}
\newcommand{\dfst}{\textsc{dfst}}
\newcommand{\mst}{\textsc{mst}}
\newcommand{\spst}{\textsc{spst}}

\newcommand{\bone}{\boldsymbol{1}}
\newcommand{\spin}{\{-1,+1\}}

\newcommand{\scO}{\mathcal{O}}
\newcommand{\jright}{\overrightarrow{v}_{\!\!k}}
\newcommand{\jrightminus}{\overrightarrow{v}_{\!\!k-1}}
\newcommand{\jleft}{\overleftarrow{v}_{\!\!k}}
\newcommand{\jleftplus}{\overleftarrow{v}_{\!\!k+1}}
\newcommand{\Cright}{\overrightarrow{c}_{\!k}}
\newcommand{\Cleft}{\overleftarrow{c}_{\!k}}
\newcommand{\anc}{\mbox{anc}}

\newcommand{\wmax}{w_{\mathrm{max}}}

\newcommand{\scN}{\mathcal{N}}

\newlength{\minipagewidth}
\newtheorem{cor}[theorem]{Corollary}

\begin{document}

\title{{\bf Random Spanning Trees and the Prediction of Weighted Graphs}}

\author{
Nicol\`o Cesa-Bianchi\\ 
Dipartimento di Informatica, Universit\`a degli Studi di Milano, Italy\\
\texttt{nicolo.cesa-bianchi@unimi.it}
\and
Claudio Gentile\\ 
DiSTA, Universit\`a dell'Insubria, Italy\\
\texttt{claudio.gentile@uninsubria.it}
\and
Fabio Vitale\\ 
Dipartimento di Informatica, Universit\`a degli Studi di Milano, Italy\\
\texttt{fabio.vitale@unimi.it}
\and
Giovanni Zappella\\
Dipartimento di Matematica, Universit\`a degli Studi di Milano, Italy\\
\texttt{giovanni.zappella@unimi.it}
}

\maketitle

\begin{abstract}
We investigate the problem of sequentially predicting the binary labels on
the nodes of an arbitrary weighted graph.
We show that, under a suitable parametrization of the problem, 
the optimal number of prediction mistakes
can be characterized (up to logarithmic factors) by the
cutsize of a random spanning tree of the graph.
The cutsize is induced by the unknown adversarial
labeling of the graph nodes.
In deriving our characterization, we obtain a simple randomized
algorithm achieving in expectation the optimal mistake bound
on any polynomially connected weighted graph.
Our algorithm draws a random spanning tree of the original
graph and then predicts the nodes of this tree in constant
expected amortized time and linear space. 
Experiments on real-world datasets show that our method compares 
well to both global (Perceptron)
and local (label propagation) methods, while being generally faster in practice.
\end{abstract}

\section{Introduction}\label{s:intro}
%
%
A widespread approach to the solution of classification problems is
representing datasets through a weighted graph where nodes 
are the data items and edge weights quantify the similarity between pairs of data items.
This technique for coding input data has been applied to several domains,
including Web spam detection~\cite{HPR09}, classification of genomic data~\cite{STS09},
face recognition~\cite{CY06}, and text categorization~\cite{GZ04}.
In many applications, edge weights are computed through a complex data-modelling
process and typically convey information that is relevant to the task of classifying the nodes.

In the sequential version of this problem, nodes are presented in an arbitrary (possibly adversarial)
order, and the learner must predict the binary label of each node before observing its true value.
Since real-world applications typically involve large datasets (i.e., large graphs), online
learning methods play an important role because of their good scaling properties.
An interesting special case of the online problem is the so-called transductive setting, where
the entire graph structure (including edge weights) is known in advance.
The transductive setting is interesting in that the learner has the chance of 
reconfiguring the graph before learning starts, so as to make the problem look easier. 
This data preprocessing can be viewed as a kind of regularization in the context
of graph prediction.

When the graph is unweighted (i.e., when all edges have the same common weight),
it was found in previous works~\cite{HPW05,HP07,Her08,HL09,CGVZ10} that
a key parameter to control the number of online prediction mistakes
is the size of the cut induced by the unknown adversarial labeling of the nodes,
i.e., the number of edges in the graph whose endpoints are assigned disagreeing labels.
However, while the number of mistakes is obviously bounded by the number of nodes, the
cutsize scales with the number of edges. This naturally led to the idea of solving
the prediction problem on a spanning tree of the graph~\cite{CGV09b,HLP09,HPR09}, whose number of edges is
exactly equal to the number of nodes minus one. Now, since the cutsize of the spanning tree
is smaller than that of the original graph, the number of mistakes in predicting the nodes
is more tightly controlled. In light of the previous discussion, we can also view the spanning tree
as a ``maximally regularized'' version of the original graph. 


Since a graph has up to exponentially many spanning trees, which one should be used
to maximize the predictive performance? This question can be answered by recalling
the adversarial nature of the online setting, where the presentation
of nodes and the assignment of labels to them are both arbitrary.
This suggests to pick a tree at random among all spanning 
trees of the graph so as to prevent the adversary from concentrating
the cutsize on the chosen tree~\cite{CGV09b}. Kirchoff's equivalence between
the effective resistance of an edge and its probability of being included
in a random spanning tree allows to express the expected cutsize of a random
spanning tree in a simple form. Namely, as the sum of resistances over
all edges in the cut of $G$ induced by the adversarial label assignment.

Although the results of \cite{CGV09b} yield a mistake bound for arbitrary unweighted 
graphs in terms of the cutsize of a random spanning tree, no general lower bounds 
are known for online unweighted graph prediction. 
The scenario gets even more uncertain in the case of weighted graphs,
where the only previous papers we are aware of
\cite{HP07,Her08,HL09} essentially contain only upper bounds.
In this paper we fill this gap, and show that the expected cutsize of a random spanning tree
of the graph delivers a convenient parametrization\footnote
{
Different parametrizations of the node prediction problem exist that lead
to bounds which are incomparable to ours ---see Section \ref{s:rel}.
} 
that captures the hardness of the graph learning problem in the general weighted case.
Given any weighted graph, we prove that any online prediction algorithm must err on
a number of nodes which is at least as big as the expected cutsize of the graph's random
spanning tree (which is defined in terms of the graph weights).
Moreover, we exhibit a simple randomized algorithm achieving in expectation the optimal mistake bound to within logarithmic factors. This bound applies to any sufficiently connected weighted graph
whose weighted cutsize is not an overwhelming fraction of the total weight.

Following the ideas of~\cite{CGV09b}, our algorithm first extracts a random spanning tree
of the original graph. Then, it predicts all nodes of this tree using a
generalization of the method proposed by \cite{HLP09}. Our tree prediction procedure
is extremely efficient: it only requires \textsl{constant} amortized time
per prediction and space \textsl{linear in the number of nodes}.
Again, we would like to stress that computational efficiency is a central issue in 
practical applications where the involved datasets can be very large. In such contexts, 
learning algorithms whose computation time scales quadratically, or slower, in the number 
of data points should be considered impractical.

As in~\cite{HLP09}, our algorithm first linearizes the tree, and then operates on
the resulting line graph via a nearest neighbor rule. We show that, besides running time, 
this linearization step brings further benefits to the overall prediction process.
In particular, similar to \cite[Theorem~4.2]{HP07}, the algorithm turns out to 
be resilient to perturbations of the labeling, a clearly desirable feature from a 
practical standpoint.

In order to provide convincing empirical evidence, we also present an
experimental evaluation of our method compared to other algorithms
recently proposed in the literature on graph prediction.
In particular, we test our algorithm against the
Perceptron algorithm with Laplacian kernel by \cite{HP07,HPR09},
and against a version of the label propagation algorithm by \cite{ZGL03}. 
These two baselines can viewed as representatives of global (Perceptron)
and local (label propagation) learning methods on graphs.
%
The experiments have been carried out on five medium-sized
real-world datasets. The two tree-based algorithms (ours and the Perceptron
algorithm) have been tested using spanning trees generated in various ways,
including committees of spanning trees aggregated by majority votes.
In a nutshell, our experimental comparison shows that predictors based on
our online algorithm compare well to all baselines while being very efficient in most cases.

The paper is organized as follows. Next, we recall preliminaries and
introduce our basic notation. Section~\ref{s:rel} surveys related work in the literature. 
In Section~\ref{s:lower} we prove the general lower bound
relating the mistakes of any prediction algorithm to the expected cutsize
of a random spanning tree of the weighted graph. In the subsequent section, we present
our prediction algorithm \wta\ (Weighted Tree Algorithm), along with a detailed mistake
bound analysis restricted to weighted trees. This analysis is extended to weighted graphs
in Section~\ref{s:gen}, where we
provide an upper bound matching the lower bound up to log factors on any sufficiently
connected graph. 
In Section~\ref{s:robust}, we quantify the robustness of our algorithm to label perturbation.
In Section~\ref{s:compl}, we provide the constant
amortized time implementation of \wta. Based on this implementation, in Section~\ref{s:exp} 
we present the experimental results. Section~\ref{s:concl} is devoted to conclusive remarks.

\subsection{Preliminaries and Basic Notation}\label{ss:not}
Let $G = (V,E,W)$ be an undirected, connected, and weighted graph with $n$ nodes and
positive edge weights $w_{i,j} > 0$ for $(i,j) \in E$.
A labeling of $G$ is any assignment $\by = (y_1,\dots,y_n) \in \spin^n$
of binary labels to its nodes. We use $(G,\by)$ to denote the resulting
labeled weighted graph.

The online learning protocol for predicting $(G,\by)$ can be defined as the following game between a (possibly randomized) learner and an adversary. The game is parameterized by the graph $G = (V,E,W)$. Preliminarly, 
and hidden to the learner, the adversary chooses a labeling $\by$ of $G$. Then the nodes of $G$ are presented 
to the learner one by one, according to a permutation of $V$, which is 
adaptively selected by the adversary. More precisely, at each time step $t=1,\dots,n$ the adversary chooses the next node $i_t$ in the permutation of $V$, and presents it to the learner for the prediction of the associated 
label $y_{i_t}$. Then $y_{i_t}$ is revealed, disclosing whether a mistake occurred. 
The learner's goal is to minimize the total number of prediction mistakes. Note that while the adversarial choice of the permutation can depend on the algorithm's randomization, the choice of the labeling is oblivious to it. In other words, the learner uses randomization to fend off the adversarial choice of labels, whereas it is fully deterministic against the adversarial choice of the permutation. The requirement that the adversary is fully oblivious when choosing labels is then dictated by the fact that the randomized learners considered in this paper make all their random choices at the beginning of the prediction process (i.e., before seeing the labels).

Now, it is reasonable to expect that prediction performance degrades
with the increase of ``randomness'' in the labeling.
For this reason, our analysis of graph prediction algorithms bounds from
above the number of prediction mistakes in terms of appropriate
notions of graph label {\em regularity}.
A standard notion of label regularity is the cutsize of a labeled graph,
defined as follows.
A $\phi$-edge of a labeled graph $(G,\by)$ is any edge $(i,j)$ such that
$y_i \neq y_j$. Similarly, an edge $(i,j)$ is $\phi$-free if $y_i = y_j$.
Let $E^{\phi} \ss E$ be the set of $\phi$-edges in $(G,\by)$.
The quantity $\Phi_G(\by) = \bigl|E^{\phi}\bigr|$ is the {\em cutsize} of $(G,\by)$,
i.e., the number of $\phi$-edges in $E^{\phi}$ (independent of the edge weights).
The {\em weighted cutsize} of $(G,\by)$ is defined by
\[
    \Phi_G^W(\by) = \sum_{(i,j)\in E^{\phi}} w_{i,j}~.
\]
For a fixed $(G,\by)$, we denote by
$r_{i,j}^W$ the effective resistance between nodes $i$ and $j$ of $G$.
In the interpretation of the graph as an electric network, where the weights
$w_{i,j}$ are the edge conductances, the effective resistance $r_{i,j}^W$ is
the voltage between $i$ and $j$ when a unit current flow is maintained through
them.
For $(i,j) \in E$, let also $p_{i,j} = w_{i,j} r_{i,j}^W$
be the probability that $(i,j)$ belongs to a random spanning tree $T$ ---see,
e.g., the monograph of~\cite{LP09}.
Then we have
\begin{equation}
\label{e:expected}
\mathbb{E}\,\Phi_T(\by) = \sum_{(i,j) \in E^{\phi}} p_{i,j} =
\sum_{(i,j) \in E^{\phi}} w_{i,j} r_{i,j}^W~,
\end{equation}
where the expectation $\mathbb{E}$ is over the random choice of spanning tree $T$.
%
%
Observe the natural weight-scale independence properties of (\ref{e:expected}).
A uniform rescaling of the edge weights $w_{i,j}$ cannot have an influence on the 
probabilities $p_{i,j}$, thereby making each product $w_{i,j} r_{i,j}^W$ scale 
independent.
In addition, since $\sum_{(i,j) \in E} p_{i,j}$
is equal to $n-1$, irrespective of the edge weighting, 
we have $0 \leq \bE\,\Phi_T(\by) \leq n-1$. Hence
the ratio $\frac{1}{n-1}\bE\,\Phi_T(\by) \in [0,1]$
provides a \textsl{density-independent} measure of the cutsize in $G$,
and even allows to compare labelings on different graphs.
%
%

Now contrast $\bE\,\Phi_T(\by)$ to the more standard weighted 
cutsize measure $\Phi_G^W(\by)$. First, $\Phi_G^W(\by)$ is clearly weight-scale dependent.
Second, it can be much larger than $n$ on dense graphs, even in the
unweighted $w_{i,j} = 1$ case. Third, it strongly depends on 
the density of $G$, which is generally related to $\sum_{(i,j) \in E} w_{i,j}$.
%
In fact, $\bE\,\Phi_T(\by)$ can be much smaller than $\Phi_G^W(\by)$
when there are strongly connected regions in $G$ contributing prominently
to the weighted cutsize. To see this, consider the following scenario:
If $(i,j) \in E^{\phi}$ and $w_{i,j}$ is large, then $(i,j)$ gives a big contribution
to $\Phi_G^W(\by)$.\footnote
{
It is easy to see that in such cases $\Phi_G^W(\by)$ can be much larger than $n$. 
} 
However, this does not necessarily happen with $\mathbb{E}\,\Phi_T(\by)$.
In fact, if $i$ and $j$ are strongly connected (i.e., if there are many disjoint paths
connecting them), then $r_{i,j}^W$ is very small and so are the terms $w_{i,j} r_{i,j}^W$
in~(\ref{e:expected}).
Therefore, the effect of the large weight $w_{i,j}$ may often be compensated by the
small probability of including $(i,j)$ in the random spanning tree.
See Figure \ref{f:1} for an example.

A different way of taking into account graph connectivity is provided by the covering
ball approach taken by \cite{Her08,HL09} --see the next section.

\begin{figure}[h!]
  \centering
\figscale{\imgdir/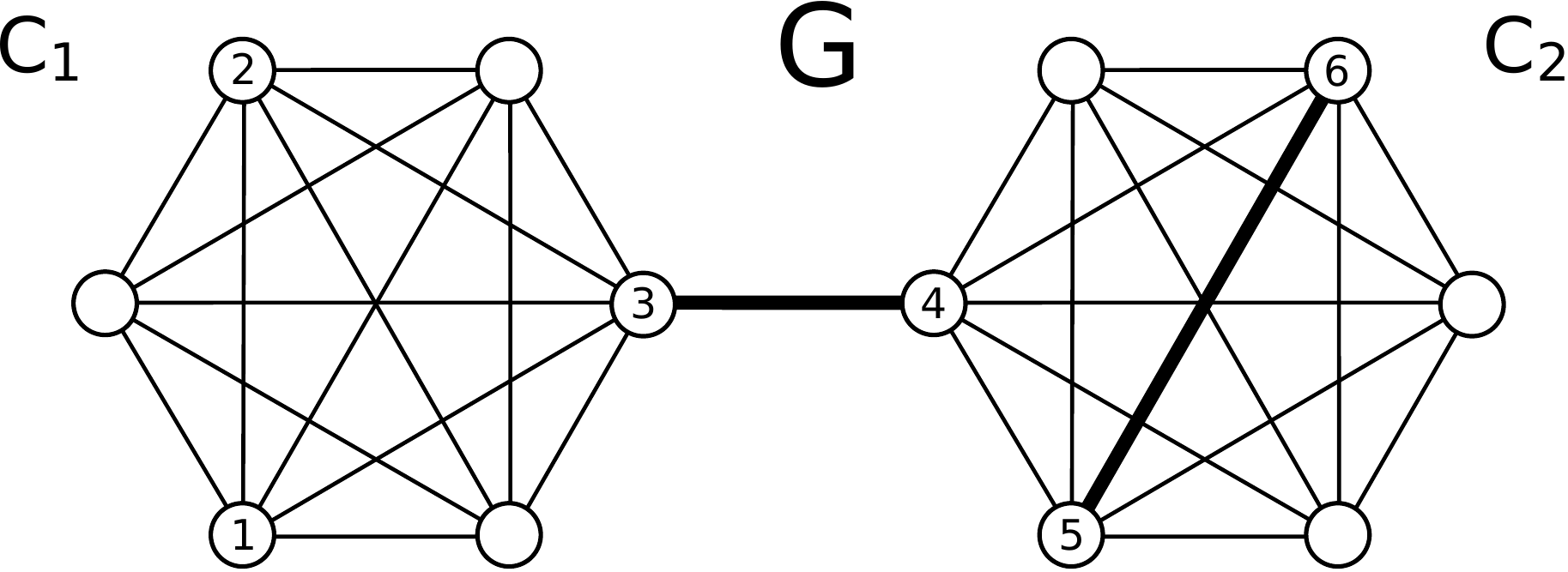}{0.52}
\caption{\label{f:1}
A barbell graph. The weight of the two thick black edges is equal to 
$\sqrt{V}$, all the other edges have unit weight. If the two labels $y_1$ and $y_2$ 
are such that $y_1 \neq y_2$, then the contribution of the edges on the left clique $C_1$ to the cutsizes 
$\Phi_G(\by)$ and $\Phi^W_G(\by)$ must be large. However, since the probability of including each edge of 
$C_1$ in a random spanning tree $T$ is $\scO(1/|V|)$, 
$C_1$'s contribution to $\bE\,\Phi_T(\by)$ is $|V|$ times smaller than 
$\Phi_{C_1}(\by) = \Phi^W_{C_1}(\by)$. 
If $y_3 \neq y_4$, then the contribution of edge (3,4) to $\Phi^W_G(\by)$ is large. 
Because this edge is a bridge, the probability of including it in $T$ is one, 
independent of $w_{3,4}$. Indeed, we have $p_{3,4} = w_{3,4}\,r^W_{3,4} = w_{3,4} / w_{3,4} = 1$. 
If $y_5 \neq y_6$, then the contribution of the right clique $C_2$ to 
$\Phi^W_G(\by)$ is large. On the other hand, the probability of including edge 
$(5,6)$ in $T$ is equal to $p_{5,6} = w_{5,6}\,r^W_{5,6} = \scO(1/\sqrt{|V|})$. Hence,
the contribution of $(5,6)$ to $\bE\,\Phi_T(\by)$ is small because the 
large weight of $(5,6)$ is offset by the fact that nodes $5$ and $6$ are strongly 
connected (i.e., there are many different paths among them). 
Finally, note that $p_{i,j} = \scO(1/|V|)$ holds 
for all edges $(i,j)$ in $C_2$, implying (similar to clique $C_1$) that 
$C_2$'s contribution to $\bE\,\Phi_T(\by)$  is $|V|$ times smaller than $\Phi^W_{C_2}(\by)$.
}
\end{figure}

\section{Related Work}\label{s:rel}
%
With the above notation and preliminaries in hand, we now briefly survey the results 
in the existing literature which are most closely related to this paper. Further comments
are made at the end of Section~\ref{s:gen}.

Standard online linear learners, such as the Perceptron algorithm, are applied
to the general
(weighted) graph prediction problem by embedding the $n$ vertices of the
graph in $\R^n$ through a map $i \mapsto K^{-1/2}\be_i$, where $\be_i \in \R^n$
is the $i$-th vector in the canonical basis of $\R^n$, and $K$ is a positive definite $n \times n$ matrix.
The graph Perceptron algorithm~\cite{HPW05,HP07} uses $K = L_G + \bone\,\bone^{\top}$,
where $L_G$ is the (weighted) Laplacian of $G$ and $\bone = (1,\dots,1)$.
The resulting mistake bound is of the form $\Phi^W_G(\by)D^W_G$,
where $D^W_G = \max_{i,j} r^W_{i,j}$ is the resistance diameter of $G$. As expected, this bound
is weight-scale independent, but the interplay between the two factors in it may lead to a vacuous
result. At a given scale for the weights $w_{i,j}$,
if $G$ is dense, then we may have $D^W_G = \scO(1)$ while $\Phi^W_G(\by)$ is of the order of $n^2$. If
$G$ is sparse, then $\Phi^W_G(\by) = \scO(n)$ but then $D^W_G$ may become as large as $n$.

The idea of using a spanning tree to reduce the cutsize of $G$ has been investigated
by~\cite{HPR09}, where the graph Perceptron algorithm is applied to a spanning tree $T$ of $G$.
The resulting mistake bound is of the form 
$\Phi^W_T(\by) D^W_T$, i.e., the graph Perceptron bound applied to tree $T$.
Since $\Phi^W_T(\by) \le \Phi^W_G(\by)$ this bound has a smaller cutsize
than the previous one. On the other hand, $D^W_T$ can be much larger than $D^W_G$ because removing edges may increase the resistance. Hence the two bounds are generally incomparable.
%
%

\cite{HPR09} suggest to apply the graph Perceptron algorithm to the spanning tree
$T$ with smallest geodesic diameter. The geodesic diameter of a weighted graph $G$
is defined by
\[
    \Delta^W_G = \max_{i,j} \min_{\Pi_{i,j}} \sum_{(r,s) \in \Pi_{i,j}} \frac{1}{w_{i,j}}
\]
where the minimum is over all paths $\Pi_{i,j}$ between $i$ and $j$.
The reason behind this choice of $T$ is that, for the spanning tree $T$ with smallest geodesic diameter,
it holds that $D^W_T \leq 2\Delta^W_G$.
However, one the one hand $D^W_G \le \Delta^W_G$, so there is no guarantee that $D^W_T=\scO\bigl(D^W_G\bigr)$, and on the other hand the adversary may still concentrate all $\phi$-edges on the chosen tree $T$, 
so there is no guarantee that $\Phi^W_T(\by)$ remains small either.

\cite{HLP09} introduce a different technique showing its application to the case of unweighted graphs.
After reducing the graph to a spanning tree $T$, the tree is linearized via a depth-first visit.
This gives a line graph $S$ (the so-called \textsl{spine} of $G$) such that $\Phi_S(\by) \le 2\,\Phi_T(\by)$.
By running a Nearest Neighbor (NN) predictor on $S$, \cite{HLP09} prove
a mistake bound of the form
$\Phi_S(\by)\log \bigl(n\big/\Phi_S(\by)\bigr) + \Phi_S(\by)$.
As observed by~\cite{FK08}, similar techniques have been developed
to solve low-congestion routing problems.

Another natural parametrization for the labels of a weighted graph that takes the graph structure into account is {\em clusterability}, i.e., the extent to which the graph nodes can be covered by a few 
balls of small resistance diameter. With this inductive bias in mind, \cite{Her08} developed the Pounce algorithm, which can be seen as a combination of graph Perceptron and NN prediction. The number of mistakes has a bound of the form
\begin{equation}\label{e:pounce}
    \min_{\rho >0}\bigl(\scN(G,\rho) + \Phi^W_G(\by)\rho\bigr)
\end{equation}
where $\scN(G,\rho)$ is the smallest number of balls of resistance diameter $\rho$ it takes to
cover the nodes of $G$.
Note that the graph Perceptron bound is recovered when $\rho = D^W_G$.
Moreover, observe that, unlike graph Perceptron's, bound (\ref{e:pounce})
is never vacuous, as it holds uniformly for all covers of $G$ (even the one made up of singletons,
corresponding to $\rho \rightarrow 0$).
A further trick for the unweighted case proposed by~\cite{HLP09} is 
to take advantage of both previous approaches (graph Perceptron and NN on line graphs) 
by building a binary tree on $G$.
This ``support tree'' helps in keeping the diameter of $G$ as small as possible, e.g., logarithmic
in the number of nodes $n$.
The resulting prediction algorithm is again a combination of a Perceptron-like algorithm and NN,
and the corresponding number of mistakes is the minimum over two earlier bounds:
a NN-based bound of the form $\Phi_G (\by)(\log n)^2$ and an unweighted version of bound (\ref{e:pounce}).

Generally speaking, clusterability and resistance-weighted cutsize $\bE\,\Phi_T(\by)$ exploit the 
graph structure in different ways. Consider, for instance, a barbell graph made up of two $m$-cliques
joined by $k$ unweighted $\phi$-edges with no endpoints in common
(hence $k \leq m$).\footnote
{
This is one of the examples considered in \cite{HL09}. 
}
If $m$ is much larger than $k$, then bound (\ref{e:pounce}) scales 
linearly with $k$ (the two balls in the cover correspond to the two $m$-cliques). On the other hand,
$\bE\,\Phi_T(\by)$ tends to be constant: Because $m$ is much larger than $k$, the probability of including 
any $\phi$-edge in $T$ tends to $1/k$, as $m$ increases and $k$ stays constant. 
On the other hand, if $k$ gets close to $m$ the resistance
diameter of the graph decreases, and (\ref{e:pounce}) becomes a constant. In fact, one can show
that when $k = m$ even $\bE\,\Phi_T(\by)$ is a constant, independent of $m$. In particular,
the probability that a $\phi$-edge is included in the random spanning tree $T$ is
upper bounded
by $\frac{3m-1}{m(m+1)}$, i.e., $\bE\,\Phi_T(\by) \rightarrow 3$ when $m$ grows large.\footnote
{
This can be shown by computing the effective resistance of $\phi$-edge $(i,j)$ as the minimum, over all
unit-strength flow functions with $i$ as source and $j$ as sink, of the squared flow values
summed over all edges, see, e.g., \cite{LP09}.
} 

When the graph at hand has a large diameter, e.g., an $m$-line graph connected to an $m$-clique 
(this is sometimes called a ``lollipop" graph) the gap between the covering-based bound (\ref{e:pounce}) 
and $\bE\,\Phi_T(\by)$ is magnified. Yet, it is fair to say that the bounds we are about to
prove for our algorithm have an extra factor, beyond $\bE\,\Phi_T(\by)$, which is
logarithmic in $m$. A similar logarithmic factor is achieved by the combined algorithm proposed 
in~\cite{HLP09}.

An even more refined way of exploiting cluster structure and connectivity in graphs is contained in
the paper of
\cite{HL09}, where the authors provide a comprehensive study of the application of dual-norm techniques 
to the prediction of weighted graphs, again with the goal of obtaining logarithmic performance guarantees
on large diameter graphs.
In order to trade-off the contribution of cutsize $\Phi^W_G$ and resistance diameter $D^W_G$,
the authors develop a notion of $p$-norm resistance.
The obtained bounds are dual norm versions of the covering ball bound ({\ref{e:pounce}}).
Roughly speaking, one can select the dual norm parameter of the algorithm to obtain a 
logarithmic contribution from the resistance diameter at the cost of squaring the contribution
due to the cutsize. This quadratic term can be further reduced if the graph is well connected.
For instance, in the unweighted barbell graph mentioned above, selecting the norm appropriately 
leads to a bound which is constant even when $k \ll m$.

Further comments on the comparison between the results presented by \cite{HL09} and the ones
in our paper are postponed to the end of Section \ref{s:gen}.

Departing from the online learning scenario, it is worth mentioning the significantly large 
literature on the general problem of learning the nodes of a graph in the train/test 
transductive setting: Many algorithms have been proposed, 
including the label-consistent mincut approach of \cite{BC01,BLRR04} and
a number of other ``energy minimization'' methods ---e.g., the ones by~\cite{ZGL03,BMN04}
of which label propagation is an instance. 
See the work of~\cite{BDL06} for a relatively recent survey on this subject.

Our graph prediction algorithm is based on a random spanning tree of the original graph.
The problem of drawing a random spanning tree of an arbitrary graph has a long history ---see, e.g., the recent monograph by~\cite{LP09}.
In the unweighted case, a random spanning tree can be sampled with a 
random walk in expected time $\mathcal{O}(n\ln n)$ for ``most'' graphs, as shown by \cite{Bro89}.
Using the beautiful algorithm of \cite{Wil96}, the expected time reduces to $\mathcal{O}(n)$ ---see also the work of~\cite{AAKKLT08}. However, all known techniques take expected time $\Theta(n^3)$ on certain pathological graphs. In the weighted case, the above methods can take longer due to the hardness of reaching, via a random walk, portions of the graph which are connected only via light-weighted edges. To sidestep this issue, in our experiments we tested a viable fast approximation where weights are disregarded when building the spanning tree, and only used at prediction time.
Finally, the space complexity for generating a random spanning tree is always linear in the graph size. 

To conclude this section, it is worth mentioning that, although we exploit random spanning trees to reduce the
cutsize, similar approaches can also be used to approximate the cutsize of a weighted graph 
by sparsification ---see, e.g., the work of~\cite{SS08}. However, because the resulting graphs are not as sparse as spanning trees, we do not currently see how to use those results.

\section{A General Lower Bound}\label{s:lower}
This section contains our general lower bound. We show that any prediction
algorithm must err at least $\tfrac{1}{2}\bE\,\Phi_T(\by)$ times on any weighted 
graph.
\begin{theorem}
\label{th:lower}
Let $G = (V,E,W)$ be a weighted undirected graph with $n$ nodes and weights
$w_{i,j} > 0$ for $(i,j) \in E$. Then for all $K \le n$ there exists a randomized
labeling $\by$ of $G$ such that for all (deterministic or randomized) algorithms $A$,
the expected number of prediction mistakes made by $A$ is at least $K/2$, while
$\bE\,\Phi_T(\by) < K$.
\end{theorem}
%
%
%

\begin{figure}
  \centering
\figscale{\imgdir/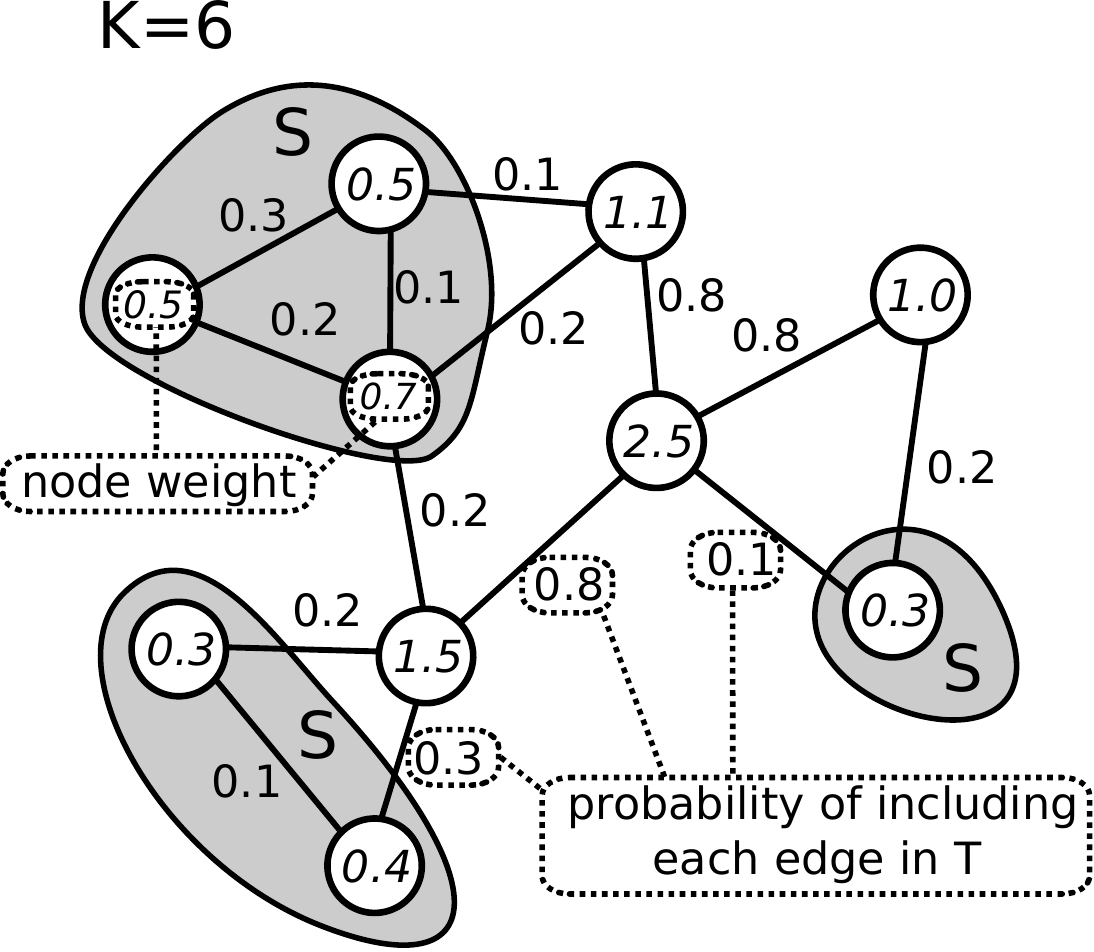}{0.52}
\caption{\label{f:2}
The adversarial strategy.
Numbers on edges are the probabilities $p_{i,j}$ of those edges being included in 
a random spanning tree for the weighted graph under consideration. Numbers within nodes denote
the weight of that node based on the $p_{i,j}$ ---see main text.
We set the budget $K$ to 6, hence the subset $S$ contains the 6 
nodes having smallest weight.
The adversary assigns a random label to each node in $S$ thus forcing 
$|S|/2$ mistakes in expectation. Then, it labels all nodes in $V \setminus S$ with 
a unique label, chosen in such a way as to minimize the cutsize consistent with the 
labels previously assigned to the nodes of $S$.
}
\end{figure}

\begin{proof}
The adversary uses the weighting $P$ induced by $W$ and defined by
$p_{i,j} = w_{i,j} r_{i,j}^W$. By~(1), 
$p_{i,j}$ is the probability
that edge $(i,j)$ belongs to a random spanning tree $T$ of $G$.
Let $P_i = \sum_{j} p_{i,j}$ be the sum over the induced weights of all
edges incident to node $i$. We call $P_i$ the \textit{weight} of node $i$.
%
%
Let $S \subseteq V$ be the set of $K$ nodes $i$ in $G$ having the smallest
weight $P_i$. The adversary assigns a random label to each node $i \in S$.
This guarantees that, no matter what, the algorithm $A$ will make on average
$K/2$ mistakes on the nodes in $S$.
The labels of the remaining nodes in $V \setminus S$ are set either all
$+1$ or all $-1$, depending on which one of the two choices yields the smaller
$\Phi_G^P(\by)$. See Figure \ref{f:2} for an illustrative example.
We now show that the weighted cutsize $\Phi_G^P(\by)$ of this
labeling $\by$ is less than $K$, \textsl{independent of} the labels of the nodes
in $S$.

%
%

Since the nodes in $V \setminus S$ have all the same label, the $\phi$-edges
induced by this labeling can only connect either two nodes in $S$ or one node
in $S$ and one node in $V \setminus S$. Hence 
$\Phi^P_G(\by)$ can be written as
\[
\Phi_G^P(\by) = \Phi^{P,\mathrm{int}}_G(\by) + \Phi^{P,\mathrm{ext}}_G(\by)
\]
where $\Phi^{P,int}_G(\by)$ is the cutsize contribution within $S$,
and $\Phi^{P,ext}_G(\by)$ is the one from edges between $S$ and 
$V \setminus S$.
We can now bound these two terms by combining the definition 
of $S$ with the equality $\sum_{(i,j) \in E} p_{i,j} = n-1$ as in the sequel.
Let
\[
    P^{\mathrm{int}}_{S} = \sum_{(i,j) \in E\,:\, i,j \in S} p_{i,j}
\qquad\text{and}\qquad
    P^{\mathrm{ext}}_{S} = \sum_{(i,j) \in E\,:\, i \in S,\, j \in V \setminus S} p_{i,j}~.
\]
From the very definition of $P^{\mathrm{int}}_{S}$ and $\Phi^{P,\mathrm{int}}_G(\by)$
we have $\Phi^{P,\mathrm{int}}_G(\by)\le P^{\mathrm{int}}_{S}$.
%
%
Moreover, from the way the labels of nodes in $V \setminus S$ are selected,
it follows that $\Phi^{P,\mathrm{ext}}_G(\by) \le P^{\mathrm{ext}}_{S}/2$. Finally, 
\[
\sum_{i \in S} P_i = 2P^{\mathrm{int}}_{S} + P^{\mathrm{ext}}_{S}
\] 
holds, since each edge
connecting nodes in $S$ is counted twice in the sum $\sum_{i \in S} P_i$.
Putting everything together we obtain
%
%
\[
2P^{\mathrm{int}}_{S} + P^{\mathrm{ext}}_{S} = \sum_{i \in S} P_i
\le \frac{K}{n}\,\sum_{i \in V} P_i 
= \frac{2K}{n}\,\sum_{(i,j)\in E} p_{i,j}
= \frac{2K(n-1)}{n}
\]
the inequality following from the definition of $S$. Hence
%
%
\[
\bE\,\Phi_T(\by) =
\Phi^P_G(\by)
= \Phi^{P,\mathrm{int}}_G(\by) + \Phi^{P,\mathrm{ext}}_G(\by)
\le
P^{\mathrm{int}}_S + \frac{P^{\mathrm{ext}}_S}{2}
\le\frac{K(n-1)}{n}
< K
\]
concluding the proof.
\end{proof}

\section{The Weighted Tree Algorithm}
\label{s:alg}
%
We now describe the Weighted Tree Algorithm ($\wta$) for predicting the labels
of a weighted tree. In Section~\ref{s:gen} we show how to apply \wta\ to 
the more general weighted graph prediction problem.
\wta\ first transforms the tree into a line graph (i.e., a list),
then runs a fast nearest neighbor method to predict the labels of each
node in the line. Though this technique is similar to that one used
by~\cite{HLP09}, the fact that the tree is weighted makes the analysis
significantly more difficult, and the practical scope of our algorithm significantly wider.
Our experimental comparison in Section \ref{s:exp} confirms that exploiting the weight 
information is often beneficial in real-world graph prediction problem.

\begin{figure}
  \centering
\figscale{\imgdir/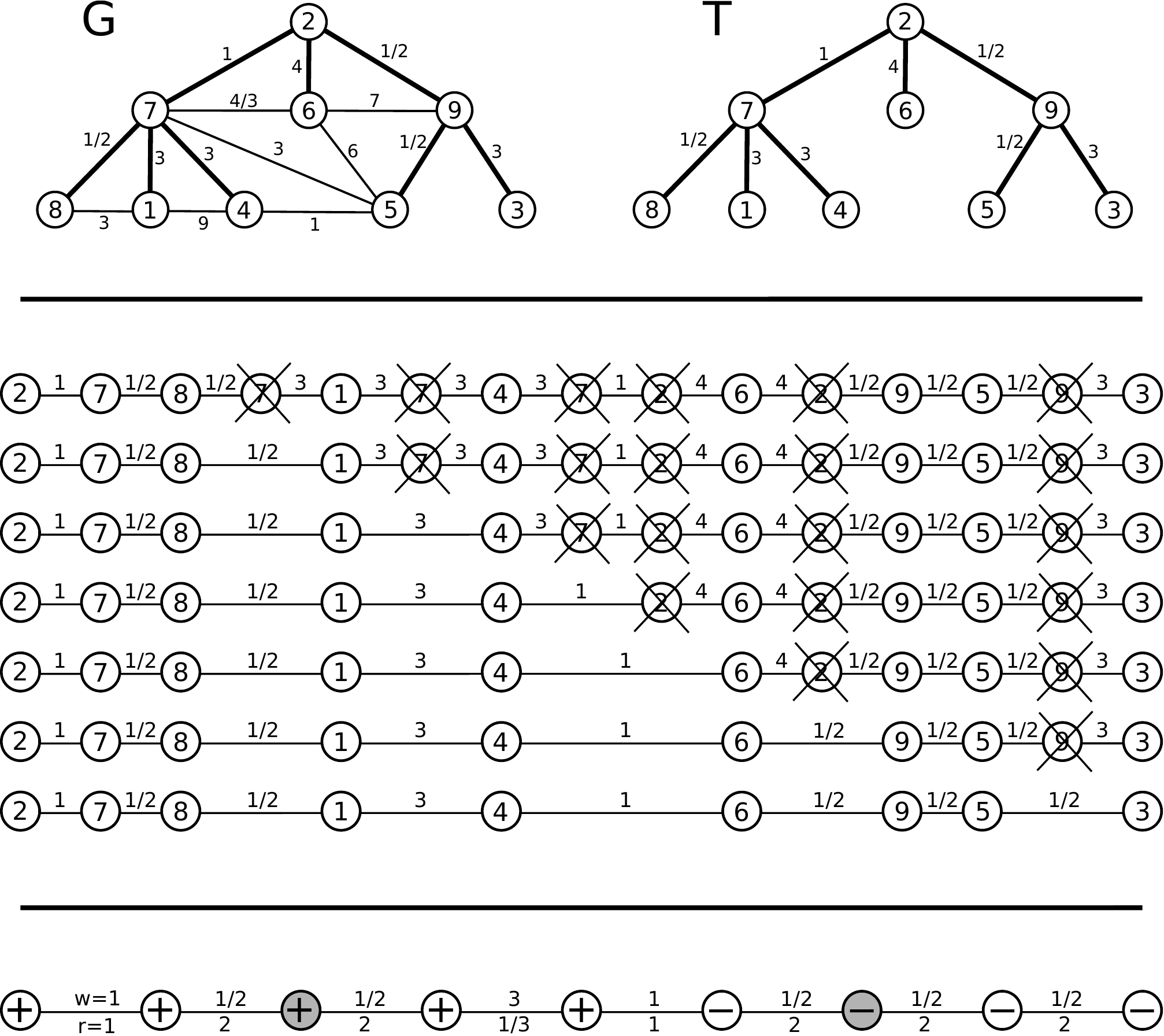}{0.52}
\caption{\label{f:3}
\textbf{Top:} A weighted graph $G$ with 9 nodes. 
Initially, \wta\ extracts a random spanning tree $T$ out of 
$G$. The weights on the edges in $T$ are the same as those of $G$.
\textbf{Middle:} The spanning tree $T$ is linearized through a depth-first traversal 
starting from an arbitrary node (node 2 in this figure). For simplicity, 
we assume the traversal visits the siblings from left to right.
As soon as a node is visited
it gets stored in a line graph $L'$ (first line graph from top). Backtracking steps 
produce duplicates in $L'$ of some of the nodes in $T$. For instance, node 7 is the
first node to be duplicated when the visit backtracks from node 8. 
The duplicated nodes are progressively eliminated from $L'$ in the order of their
insertion in $L'$.
Several iterations of this node elimination process are displayed from the top to
the bottom, showing how $L'$ is progressively shrunk to the final line $L$ (bottom line). 
Each line represents the elimination of a single duplicated node.
The crossed nodes in each line are the nodes which are scheduled
to be eliminated. Each time a new node $j$ is eliminated, its two adjacent nodes $i$ and
$k$ are connected by the lighter 
of the two edges $(i,j)$ and $(j,k)$. For instance: the left-most duplicated 7 
is dropped by directly connecting the two adjacent nodes 8 and 1 by an edge with 
weight $1/2$; the right-most node 2 is eliminated by directly connecting node 6 to 
node 9 with an edge with weight $1/2$, and so on.
Observe that this elimination procedure can be carried out 
{\em in any order} without changing the resulting list $L$.  
\textbf{Bottom:} We show \wta's prediction on the line $L$ so obtained. 
In this figure, the numbers above the edges 
denote the edge weights, the ones below are the resistors, i.e., weight reciprocals. 
We are at time step $t=3$ where two labels have so far been revealed (gray nodes).
\wta\ predicts on the remaining nodes according to a nearest neighbor rule on $L$, 
based on the resistance distance metric. All possible predictions made by \wta\ at this
time step are shown.
}
\end{figure}

Given a labeled weighted tree $(T,\by)$, the algorithm initially creates a weighted line
graph $L'$ containing some duplicates of the nodes in $T$. Then, each duplicate node
(together with its incident edges) is replaced by a single edge with a suitably chosen weight.
This results in the final weighted line graph $L$ which is then used for prediction.
%
In order to create $L$ from $T$, $\wta$ performs the
following \textsl{tree linearization} steps:
%
\begin{enumerate}
\item An arbitrary node $r$ of $T$ is chosen, and a line $L'$ containing only $r$ is created.
\item Starting from $r$, a depth-first visit of $T$ is performed.
Each time an edge $(i,j)$ is traversed (even in a backtracking step) from $i$ to $j$,
the edge is appended to $L'$ with its weight $w_{i,j}$, and $j$ becomes the current
terminal node of $L'$.
Note that backtracking steps can create in $L'$ at most one duplicate of each edge
in $T$, while nodes in $T$ may be duplicated several times in $L'$.
\item $L'$ is traversed once, starting from terminal $r$. During this traversal,
duplicate nodes are eliminated as soon as they are encountered. This works as follows.
Let $j$ be a duplicate node, and $(j',j)$ and $(j,j'')$ be the two incident
edges. The two edges are replaced by a new edge $(j',j'')$ having weight
$w_{j',j''} = \min\bigl\{w_{j',j}, w_{j,j''}\bigr\}$.\footnote
{
By iterating this elimination procedure, it might happen that more than two adjacent 
nodes get eliminated. In this case, the two surviving terminal nodes are connected 
in $L$ by the lightest edge among the eliminated ones in $L'$.
}
Let $L$ be the resulting line.
\end{enumerate}
The analysis of Section~\ref{ss:analysis} shows that this choice of $w_{j',j''}$
guarantees that the weighted cutsize of $L$ is smaller than twice the
weighted cutsize of $T$.
%

Once $L$ is created from $T$, the algorithm predicts the label 
of each node $i_t$ using a nearest-neighbor rule operating on $L$ with
a {\em resistance distance} metric. That is, the prediction on $i_t$ is the label of $i_{s^*}$, being
$s^* = \argmin_{s < t} d(i_s,i_t)$ the previously revealed node closest
to $i_t$, and
$
    d(i,j) = \sum_{s = 1}^k {1}/{w_{v_s,v_{s+1}}}
$
is the sum of the resistors (i.e., reciprocals of edge weights)
along the unique path $i = v_1 \rightarrow v_2 \rightarrow\cdots\rightarrow v_{k+1} = j$ 
connecting node $i$ to node $j$.
Figure \ref{f:3} gives an example of \wta\ at work.

\subsection{Analysis of WTA}\label{ss:analysis}
%
The following lemma gives a mistake bound on \wta\ run on any weighted line graph.
Given any labeled graph $(G,\by)$, we denote by $R^W_G$ the sum of resistors
of $\phi$-free edges in $G$,
\[
R^W_G = \sum_{(i,j)\in E \setminus E^{\phi}} \frac{1}{w_{i,j}}\ .
\]
Also, given any $\phi$-free edge subset $E' \subset E \setminus E^{\phi}$, we 
define $R^W_{G}(\neg E')$ as the sum of the resistors of 
all $\phi$-free edges in $E \setminus (E^{\phi} \cup E')$,
\[
R^W_{G}(\neg E') = \sum_{(i,j) \in E \setminus (E^{\phi} \cup E')} \frac{1}{w_{i,j}}\ .
\]
Note that $R^W_{G}(\neg E') \le R^W_G$, since we drop some edges from the sum in the defining formula.

Finally, we use $f \bigoheq g$ as shorthand for $f = \scO(g)$. 
The following lemma is the starting point of our theoretical investigation ---please
see Appendix~A for proofs.

\begin{lemma}\label{l:ub-L-to-T}
If \wta\ is run on a labeled weighted line graph $(L,\by)$, then the total number
$m_L$ of mistakes satisfies
\[
    m_L
\bigoheq
    \Phi_L(\by)
    \left(1 + \log\left(1+\frac{R^W_L(\neg E') \ \Phi^W_L(\by)}{\Phi_L(\by)}\right)
    \right)+|E'|
\]
for all subsets $E'$ of $E \setminus E^{\phi}$.
\end{lemma}
Note that the bound of Lemma~\ref{l:ub-L-to-T} implies that, for any $K=|E'| \geq 0$,
one can drop from the bound the contribution of any set of $K$ resistors in $R^W_L$
at the cost of adding $K$ extra mistakes.
We now provide an upper bound on the number of mistakes made by \wta\ on any
weighted tree $T = (V,E,W)$ in terms of the number of $\phi$-edges, the
weighted cutsize, and $R^W_T$.

\begin{theorem}\label{t:ub-tree}
If \wta\ is run on a labeled weighted tree $(T,\by)$, then the total number
$m_T$ of mistakes satisfies
\[
    m_T
\bigoheq
    \Phi_T(\by) \left(1 + \log\left(1+\frac{R^W_T(\neg E') \ \Phi_T^W(\by)}{\Phi_T(\by)}\right)
     \right) + |E'|
\]
for all subsets $E'$ of $E \setminus E^{\phi}$.
\end{theorem}
The logarithmic factor in the above bound shows 
that the algorithm takes advantage of labelings such that the weights of
$\phi$-edges are small (thus making $\Phi_T^W(\by)$ small) and the
weights of $\phi$-free edges are high (thus making $R^W_T$ small).
This matches the intuition behind $\wta$'s nearest-neighbor rule
according to which nodes that are close to each other are expected to have the same
label. In particular, observe that the way the above quantities are combined makes 
the bound independent of rescaling of the edge weights.
Again, this has to be expected, since $\wta$'s prediction is scale insensitive.
On the other hand, it may appear less natural that the mistake bound
also depends linearly on the cutsize $\Phi_T(\by)$, {\em independent of
the edge weights}. The specialization to trees of our lower bound (Theorem~\ref{th:lower}
in Section~\ref{s:lower}) implies that this linear dependence of mistakes on the unweighted cutsize is necessary
whenever the adversarial labeling is chosen from a set of labelings with bounded $\Phi_T(\by)$.

\section{Predicting a Weighted Graph}
\label{s:gen}
In order to solve the more general problem of predicting the labels 
of a weighted graph $G$, one can first generate a spanning
tree $T$ of $G$ and then run \wta\ directly on $T$. In this case,
it is possible to rephrase Theorem~\ref{t:ub-tree} in terms of the properties
of $G$. Note that for each spanning tree $T$ of $G$,
$\Phi_T^W(\by) \le \Phi_G^W(\by)$ and $\Phi_T(\by) \le \Phi_G(\by)$.
Specific choices of the spanning tree $T$ control in different ways
the quantities in the mistake bound of Theorem~\ref{t:ub-tree}.
For example, a minimum spanning tree tends to reduce the value of $\wt{R}^W_T$,
betting on the fact that $\phi$-edges are light. The next theorem relies
on {\em random} spanning trees.
%
%
%
%
\begin{theorem}
\label{th:graph}
If \wta\ is run on a random spanning tree $T$ of a labeled weighted graph
$(G,\by)$, then the total number $m_G$ of mistakes satisfies
\begin{equation}
\label{e:th_general_graphs}
    \bE\,m_G
\bigoheq
    \bE\bigl[\Phi_T(\by)\bigr]\Bigl(1 + \log \left(1 + \wmax^{\phi} \bE\bigl[R^W_T\bigr] \Bigr) \right)~,
\end{equation}
where ${\dt \wmax^{\phi} = \max_{(i,j) \in E^{\phi}} w_{i,j} }$.
\end{theorem}
Note that the mistake bound in~(\ref{e:th_general_graphs})
is scale-invariant, since 
$\bE\bigl[\Phi_T(\by)\bigr]=\sum_{(i,j) \in E^{\phi}} w_{i,j} r_{i,j}^W$
cannot be affected by a uniform rescaling of the edge weights (as we said in Subsection~\ref{ss:not}), and so is
the product $w^{\phi}_{\max}\bE\bigl[R^W_T\bigr]=w^{\phi}_{\max}\sum_{(i,j) \in E 
\setminus E^{\phi}} r^W_{i,j}$. 

We now compare the mistake bound~(\ref{e:th_general_graphs}) 
to the lower bound stated in Theorem~\ref{th:lower}. In particular, we prove
that $\wta$ is optimal (up to $\log n$ factors) on every
weighted connected graph in which the $\phi$-edge weights 
are not ``superpolynomially overloaded'' w.r.t.\ the $\phi$-free edge weights.
In order to rule out pathological cases, when the weighted graph
is nearly disconnected, we impose the following mild assumption
on the graphs being considered.

We say that a graph is \textsl{polynomially connected} if the ratio of any pair of effective 
resistances (even those between nonadjacent nodes) in the graph is polynomial in the 
total number of nodes $n$.   
This definition essentially states that a weighted graph can be 
considered connected if no pair of nodes can be found which is substantially
less connected than any other pair of nodes. Again, as one would naturally expect,
this definition is independent of uniform weight rescaling.  
The following corollary shows that if \wta\ is not optimal on a polynomially connected
graph, then the labeling must be so irregular that the total weight of $\phi$-edges
is an overwhelming fraction of the overall weight.
%
\begin{cor}
\label{cor:upper}
Pick any polynomially connected weighted graph $G$ with $n$ nodes.
If the ratio of the total weight of $\phi$-edges to the total weight
of $\phi$-free edges is bounded by a polynomial in $n$, then
the total number of mistakes $m_G$ made by \wta\ 
when run on a random spanning tree $T$ of G
satisfies $\bE\,m_G \bigoheq \bE \bigl[\Phi_T(\by)\bigr] \log n$. 
\end{cor}
Note that when the hypothesis of this corollary is not satisfied
the bound of \wta\ is not necessarly vacuous.
For example, $\bE\bigl[R^W_T\bigr]w^{\phi}_{\max} = n^{\mathrm{polylog}(n)}$
implies an upper bound which is optimal up to $\mathrm{polylog}(n)$ factors.
%
%
%
%
%
In particular, having a constant number of $\phi$-free edges with exponentially large 
resistance contradicts the assumption of polynomial connectivity, but it need not
lead to a vacuous bound in Theorem~\ref{th:graph}.
In fact, one can use Lemma~\ref{l:ub-L-to-T}
to drop from the mistake bound of Theorem~\ref{th:graph} 
the contribution of any set of $\scO(1)$ resistances in 
$\bE\bigl[R^W_T\bigr]=\sum_{(i,j) \in E \setminus E^{\phi}} r^W_{i,j}$ at the
cost of adding just $\scO(1)$ extra mistakes.
This could be seen as a robustness property of $\wta$'s bound 
against graphs that do not fully satisfy the connectedness assumption.

We further elaborate on the robustness properties of \wta\ in Section~\ref{s:robust}.
In the meanwhile, note how Corollary~\ref{cor:upper} compares to the expected mistake bound
of algorithms 
like graph Perceptron (see Section~\ref{s:rel}) 
on the same random spanning tree. 
This bound depends on the
expectation of the product $\Phi^W_T(\by)D^W_T$, where $D^W_T$ is the
diameter of $T$ in the resistance distance metric. Recall from the discussion in
Section~\ref{s:rel} that these two factors
are negatively correlated because $\Phi_T^W(\by)$ depends linearly on the edge
weights, while $D^W_T$ depends linearly on the reciprocal of these weights.
Moreover, for any given scale of the edge weights, $D^W_T$ can be linear in the
number $n$ of nodes.

Another interesting comparison is to the covering ball bounds of~\cite{Her08,HL09}.
Consider the case when $G$ is an unweighted tree with diameter $D$. 
Whereas the dual norm approach of~\cite{HL09} gives a mistake bound of the form 
$\Phi_G(\by)^2\,\log D$, our approach, as well as the one by~\cite{HLP09},
yields  $\Phi_G(\by)\,\log n$. Namely, the dependence on $\Phi_G(\by)$
becomes linear rather than quadratic,
but the diameter $D$ gets replaced by $n$, the number of nodes in $G$.
Replacing $n$ by $D$ seems to be a benefit brought by the covering ball approach.\footnote
{
As a matter of fact, a bound of the form $\Phi_G(\by)\,\log D$ on unweighted trees is 
also achieved by the direct analysis of~\cite{CGV09b}.
}  
More generally, one can say that the covering ball approach seems to allow to replace
the extra $\log n$ term contained in Corollary~\ref{cor:upper} by more refined 
structural parameters of the graph (like its diameter $D$), but it does so at the cost of
squaring the dependence on the cutsize. 
A typical (and unsurprising) example where the dual-norm covering ball bounds are better 
then the one in Corollary~\ref{cor:upper} is when the labeled graph is
well-clustered. One such example we already mentioned in Section~\ref{s:rel}:
On the unweighted barbell graph made up of $m$-cliques connected by $k \ll m$ $\phi$-edges, 
the algorithm of~\cite{HL09} has a {\em constant} 
bound on the number of mistakes (i.e., independent of both $m$ and $k$), the Pounce algorithm
has a {\em linear} bound in $k$, while Corollary \ref{cor:upper} delivers a {\em logarithmic}
bound in $m+k$. Yet, it is fair to point out that the bounds of~\cite{Her08,HL09}
refer to computationally heavier algorithms than \wta: Pounce has a deterministic initialization step that computes 
the inverse Laplacian matrix of the graph (this is cubic in $n$, or quadratic in the case of trees), the minimum $(\Psi,p)$-seminorm interpolation algorithm of~\cite{HL09} has no initialization,
but each step requires the solution of a constrained convex optimization problem (whose 
time complexity was not quantified by the authors).
Further comments on the time complexity of our algorithm are given in Section \ref{s:compl}.

\section{The Robustness of WTA to Label Perturbation}
\label{s:robust}
In this section we show that \wta\ is tolerant to noise, i.e., the number of mistakes made by \wta\ 
on most labeled graphs $(G,\by)$ does not significantly change 
if a small number of labels are perturbed 
before running the algorithm. This is especially the case if the input 
graph $G$ is polynomially connected (see Section~\ref{s:gen} for a definition).

As in previous sections, we start off from the case when the input 
graph is a tree, and then we extend the result to general graphs using random spanning trees.


Suppose that the labels $\by$ in the tree $(T,\by)$ used as input to the 
algorithm have actually been obtained from another labeling $\by'$ of $T$
through the perturbation (flipping) of some of its labels.
As explained at the beginning of Section~\ref{s:alg},
\wta\ operates on a line graph $L$ obtained through the linearization process
of the input tree $T$. The following theorem shows that, whereas the cutsize differences
$|\Phi^W_T(\by)-\Phi^W_T(\by')|$ and $|\Phi_T(\by)-\Phi_T(\by')|$ on tree $T$ can in principle
be very large, the cutsize differences
$|\Phi^W_L(\by)-\Phi^W_L(\by')|$ and 
$|\Phi_L(\by)-\Phi_L(\by')|$  on the line graph $L$ built by \wta\ are
always small.

In order to quantify the above differences, we need a couple of ancillary
definitions. Given a labeled tree $(T,\by)$, define $\zeta_T(K)$ to be
the sum of the weights of the $K$ heaviest edges 
in $T$,
\[
\zeta_T(K) = \max_{E' \subseteq E \,:\, |E'| = K} \sum_{(i,j) \in E'} w_{i,j}\ .
\]
If $T$ is unweighted we clearly have $\zeta_T(K)=K$.
Moreover, given any two labelings $\by$ and $\by'$ of $T$'s nodes, we let 
$\delta(\by,\by')$ be the number of nodes for which the two labelings
differ, i.e.,
\(
\delta(\by,\by') = \bigl|\{i=1,\dots,n \,:\, y_i\neq y_i'\}\bigr|\ .
\)
%
%
%
%
%
%
\begin{theorem}\label{t:robust}
On any given labeled tree $(T,\by)$ the tree linearization step of \wta\ 
generates a line graph $L$ such that:
\begin{enumerate}
\item ${\dt \Phi_L^W(\by) \le \min_{\by' \in \{-1,+1\}^n} 2\Bigl(\Phi^W_T(\by')+\zeta_T\bigl({\delta(\by,\by')}\bigr)\Bigr)}$~;
\item ${\dt \Phi_L(\by) \le \min_{\by' \in \{-1,+1\}^n} 2\left(\Phi_T(\by')+\delta(\by,\by')\right) }$~.
\end{enumerate}
\end{theorem}
%
%
In order to highlight the consequences of \wta's linearization step 
contained in Theorem~\ref{t:robust}, consider as a simple example an unweighted 
star graph $(T,\by)$ where all labels are $+1$ except for the central
node $c$ whose label is $-1$. We have $\Phi_T(\by)= n-1$, but
flipping the sign of $y_c$ we would obtain the star graph $(T,\by')$ with $\Phi_T(\by')=0$.
Using Theorem~\ref{t:robust} (item 2) we get $\Phi_L(\by) \leq 2$.
Hence, on this star graph \wta's linearization step
generates a line graph with a constant number of 
$\phi$-edges even if the input tree $T$ has no $\phi$-free edges. 
Because flipping the labels of a few nodes 
(in this case the label of $c$) we obtain a tree
with a much more regular labeling, the labels of those nodes 
can naturally be seen as corrupted by noise.


The following theorem quantifies to what extent the mistake bound
of \wta\ on trees can take advantage of the tolerance to label
perturbation contained in Theorem~\ref{t:robust}. 
Introducing shorthands for the right-hand side expressions in Theorem~\ref{t:robust},
\[
\wt{\Phi}^W_T(\by) = \min_{\by' \in \{-1,+1\}^n} 2\Bigl(\Phi^W_T(\by')+\zeta_T\bigl({\delta(\by,\by')}\bigr)\Bigr)
\]
and
\[
\wt{\Phi}_T(\by) = \min_{\by' \in \{-1,+1\}^n} 2\left(\Phi_T(\by')+\delta(\by,\by')\right)~,
\]
we have the following robust version of Theorem~\ref{t:ub-tree}.
\begin{theorem}\label{t:ub-tree-r}
If \wta\ is run on a weighted and labeled tree $(T,\by)$, then the total number
$m_T$ of mistakes satisfies
\[
    m_T
\bigoheq
    \wt{\Phi}_T(\by) \left(1 + \log\left(1+\frac{R^W_T(\neg E') \ \wt{\Phi}_T^W(\by)}{\wt{\Phi}_T(\by)}\right)
     \right) + \Phi_T(\by) + |E'|
\]
for all subsets $E'$ of $E \setminus E^{\phi}$.
\end{theorem}
As a simple consequence, we have the following corollary.
\begin{corollary}\label{c:ub-tree-r-poly}
If \wta\ is run on a weighted and polynomially connected labeled tree $(T,\by)$, then the total number
$m_T$ of mistakes satisfies
\[
    m_T
\bigoheq
\wt{\Phi}_T(\by)  \log n~.
\]
\end{corollary}
Theorem\ \ref{t:ub-tree-r} combines the result of Theorem~\ref{t:ub-tree}
with the robustness to label perturbation of \wta's tree linearization procedure.
Comparing the two theorems, we see that the 
main advantage of the tree linearization lies in the mistake bound dependence on the logarithmic
factors occurring in the formulas: Theorem~\ref{t:ub-tree-r} 
shows that, when $\wt{\Phi}_T(\by) \ll \Phi_T(\by)$, then the performance of \wta\ can be just 
linear in $\Phi_T(\by)$. Theorem~\ref{t:ub-tree}
shows instead that the dependence on 
$\Phi_T(\by)$ is in general superlinear even in cases when flipping few labels of $\by$ makes the 
cutsize $\Phi_T(\by)$ decrease in a substantial way.
In many cases, the tolerance to noise allows us to achieve even better results: 
Corollary~\ref{c:ub-tree-r-poly} states that, if $T$ is polynomially connected and 
there exists a labeling $\by'$ with small $\delta(\by,\by')$ such that 
$\Phi_T(\by')$ is much smaller than $\Phi_T(\by)$, then the performance of \wta\
is about the same as if the algorithm were run on $(T,\by')$. 
In fact, from Lemma~\ref{l:ub-L-to-T}
we know that when $T$ is polynomially connected the mistake bound of \wta\ mainly depends
on the number of $\phi$-edges in $(L,\by)$, which
can often be much smaller than those in $(T,\by)$. As a simple example,
let $T$ be an unweighted star graph with a labeling $\by$ and
$z$ be the difference between the number of $+1$ and the number of $-1$ in $\by$.
Then the mistake bound of $\wta$ is linear in $z \log n$ irrespective
of $\Phi_T(\by)$ and, specifically, irrespective of the label assigned to the central
node of the star, which can greatly affect the actual value of $\Phi_T(\by)$.

We are now ready to extend the above results to the case
when \wta\ operates on a general weighted graph $(G, \by)$ via
a uniformly generated random spanning tree $T$.
%
%
As before, we need some shorthand notation. Define $\Phi_G^*(\by)$ as
\[
\Phi_G^*(\by)  =   \min_{\by' \in \{-1, +1\}^n} \Bigl(\bE\bigl[\Phi_T(\by')\bigr]+\delta(\by,\by')\Bigr)~,
\]
where the expectation is over the random draw of a spanning tree $T$ of $G$.
The following are the robust versions of Theorem~\ref{th:graph} and Corollary~\ref{cor:upper}.
\begin{theorem}
\label{th:graph_r}
If \wta\ is run on a random spanning tree $T$ of a labeled weighted graph
$(G,\by)$, then the total number $m_G$ of mistakes satisfies
\[
    \bE\,m_G
\bigoheq
    \Phi_G^*(\by)\Bigl(1 + \log \left(1 + \wmax^{\phi} \bE\bigl[R^W_T\bigr] \Bigr) \right)+\bE\bigl[\Phi_T(\by)\bigr]~,
\]
where ${\dt \wmax^{\phi} = \max_{(i,j) \in E^{\phi}} w_{i,j} }$.
\end{theorem}
\begin{corollary}
\label{c:graph_pol_r}
If \wta\ is run on a random spanning tree $T$ of a labeled weighted graph
$(G,\by)$ and the ratio of the weights of each pair of edges of $G$ is polynomial in $n$, 
then the total number $m_G$ of mistakes satisfies
\[
    \bE\,m_G
\bigoheq
    \Phi_G(\by)\log n~.
\]
\end{corollary}
The relationship between Theorem~\ref{th:graph_r} and Theorem~\ref{th:graph}
is similar to the one between Theorem~\ref{t:ub-tree-r} and Theorem~\ref{t:ub-tree}.
When there exists a labeling 
$\by'$ such that $\delta(\by,\by')$ is small and $\bE\bigl[\Phi_T(\by')\bigr] \ll \bE\bigl[\Phi_T(\by)\bigr]$, 
then Theorem~\ref{th:graph_r} allows a linear dependence on $\bE\bigl[\Phi_T(\by)\bigr]$.
Finally, Corollary~\ref{c:graph_pol_r} quantifies the advantages of \wta's noise tolerance under a similar (but stricter) assumption as the one contained in Corollary~\ref{cor:upper}.

\section{Implementation}\label{s:compl}
As explained in Section~\ref{s:alg}, $\wta$ runs in two phases: (i) a random spanning tree is drawn; (ii) the tree is linearized and labels are sequentially predicted. As discussed in Subsection~\ref{ss:not}, Wilson's algorithm can draw a random spanning tree of ``most'' unweighted graphs in expected time $\mathcal{O}(n)$. The analysis of running times on weighted graphs is significantly more complex, and outside the scope of this paper. A naive implementation of \wta's
second phase runs in time $\mathcal{O}(n \log n)$ and requires linear memory space when operating on a tree with $n$ nodes. We now describe how to implement the second phase to run in time $\mathcal{O}(n)$, i.e., in {\em constant} amortized time per prediction step.

Once the given tree $T$ is linearized into an $n$-node line $L$, 
we initially traverse $L$ from left to right. Call $j_0$ the left-most terminal node of $L$.
During this traversal, the resistance distance  $d(j_0,i)$ is incrementally computed for each node 
$i$ in $L$. This makes it possible to calculate $d(i,j)$ in constant time for any pair of nodes, 
since $d(i,j) = |d(j_0,i)-d(j_0,j)|$ for all $i,j \in L$.
On top of $L$, a complete binary tree $T'$ with $2^{\lceil \log_2 n \rceil}$ leaves
is constructed.\footnote
{
For simplicity, this description assumes $n$ is a power of $2$. 
If this is not the case, we could add dummy nodes to $L$
before building $T'$. 
}
The $k$-th leftmost leaf (in the usual tree representation) 
of $T'$ is the $k$-th node in $L$ (numbering the nodes of $L$ from left to right).
The algorithm maintains this data-structure in such a way that at time $t$: 
(i) the subsequence of leaves whose labels are revealed at time $t$ are connected through 
a (bidirectional) list $B$, and
(ii) all the ancestors in $T'$ of the leaves of $B$ are marked.
See Figure~\ref{f:fast-impl}.

\begin{figure}
  \centering
%
\figscale{\imgdir/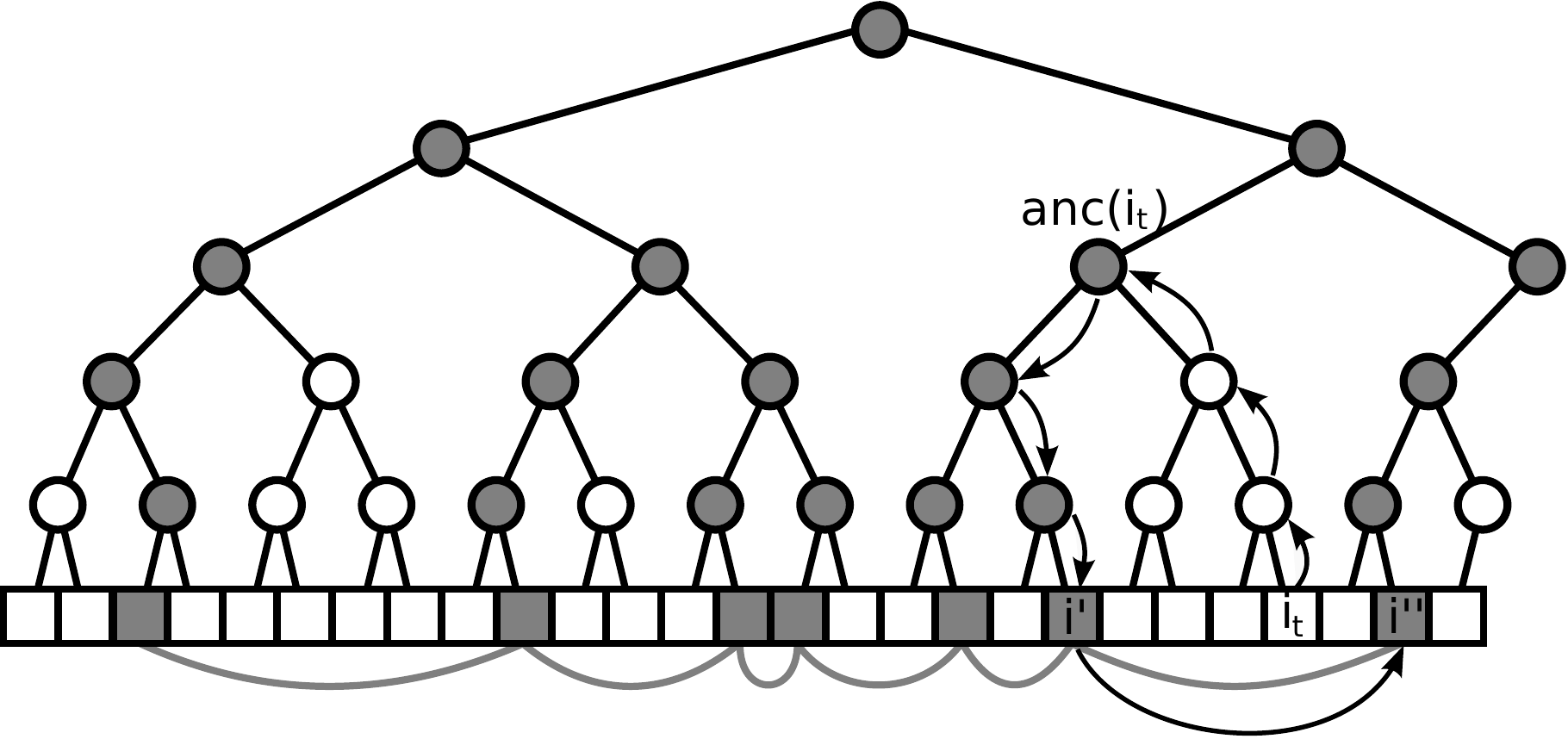}{0.60}
%
\caption{\label{f:fast-impl}
Constant amortized time implementation of $\wta$. The line $L$ has $n = 27$ nodes
(the adjacent squares at the bottom).
Shaded squares are the revealed nodes, connected through a dark grey doubly-linked list $B$.
The depicted tree $T'$ has both unmarked (white) and marked (shaded) nodes. The arrows indicate the
traversal operations performed by $\wta$ when predicting the label of node $i_t$: The upward traversal
stops as soon as a marked ancestor $\anc(i_t)$ is found, and then a downward traversal begins. Note
that $\wta$ first descends to the left, and then keeps going right all the way down.
Once $i'$ is determined, a single step within $B$ suffices to determine $i''$.
}
\end{figure}
When $\wta$ is required to predict the label $y_{i_t}$, the algorithm looks for the 
two closest revealed leaves $i'$ and $i''$ oppositely located in $L$ with respect to $i_t$. 
The above data structure supports this operation as follows.
$\wta$ starts from $i_t$ and goes upwards in $T'$ until the first marked ancestor $\anc(i_t)$ of 
$i_t$ 
is reached. During this upward traversal, the algorithm 
marks each internal node of $T'$ on the path connecting $i_t$ to $\anc(i_t)$. Then, $\wta$ 
starts from $\anc(i_t)$ and goes downwards in order to find the leaf $i' \in B$ closest to $i_t$. 
Note how the algorithm uses node marks for finding its way down: For instance,
in Figure~\ref{f:fast-impl} the algorithm goes left since $\anc(i_t)$ was reached from below through
the right child node, and then keeps right all the way down to $i'$.
Node $i''$ (if present) is then identified via the links in $B$. The two distances
$d(i_t,i')$ and $d(i_t,i'')$ are compared, and the closest node to $i_t$ within $B$ is then determined.
Finally, $\wta$ updates the links of $B$ by inserting $i_t$ between $i'$ and $i''$.

In order to quantify the amortized time per trial, the key observation is that
each internal node $k$ of $T'$ gets visited only twice during {\em upward} traversals 
over the $n$ trials:
The first visit takes place when $k$ gets marked for the first time,
the second visit of $k$ occurs when a subsequent upward visit also marks the other
(unmarked) child of $k$. Once both of $k$'s children are marked, we are guaranteed that
no further upward visits to $k$ will be performed.
Since the preprocessing operations take $\mathcal{O}(n)$, this shows that 
the total running time over the $n$ trials is linear 
in $n$, as anticipated. 
Note, however, that the worst-case time per trial is $\mathcal{O}(\log n)$.
For instance, on the very first trial $T'$ has to be traversed all the way up and down.

This is the way we implemented \wta\ on the experiments described in the next section.

\section{Experiments}\label{s:exp}
We now present the results of an experimental comparison on a number of
real-world weighted graphs from different domains:
text categorization, optical character recognition, spam detection and bioinformatics.
Although our theoretical analysis is for the sequential prediction model, all experiments
are carried out using a more standard train-test scenario. This makes it easy to compare
\wta\ against popular non-sequential baselines, such as Label Propagation.

We compare our algorithm to the following other methods, intended as 
representatives of two different ways of coping with the graph prediction problem:
global vs.\ local prediction.

\paragraph{Perceptron with Laplacian kernel.}
Introduced by \cite{HP07} and here abbreviated as $\gpa$ (graph Perceptron algorithm).
This algorithm sequentially predicts the nodes of a weighted graph $G = (V,E)$ after
mapping $V$ via the linear kernel based on $L_G^{+} + \bone\,\bone^{\top}$, where 
$L_G$ is the laplacian matrix of $G$. Following \cite{HPR09}, we run $\gpa$ on a
spanning tree $T$ of the original graph. This is because a careful computation of the Laplacian
pseudoinverse of a $n$-node tree takes time $\Theta(n + m^2 + mD)$ where $m$ is the number of 
training examples plus the number of test examples (labels to predict), and $D$ is the tree 
diameter ---see the work of~\cite{HPR09} for a proof of this fact. However, in most of our 
experiments $m=n$, implying a running time of $\Theta(n^2)$ for $\gpa$. 

Note that $\gpa$ is a global approach, in that the graph topology affects, 
via the inverse Laplacian, the prediction on all nodes. 

\paragraph{Weighted Majority Vote.} Introduced here and abbreviated as $\omv$.
Since the common underlying assumption to graph prediction algorithms is that adjacent
nodes are labeled similarly, a very intuitive and fast algorithm 
for predicting the label of a node $i$ is via a weighted majority vote on 
the available labels of the adjacent nodes. More precisely, \omv\ predicts using the sign of
\[
    \sum_{j \,:\, (i,j) \in E} y_{j} w_{i,j}
\]
where $y_{j} = 0$ if node $j$ is not available in the training set.
The overall time and space requirements are both of order $\Theta(|E|)$, since we 
need to read (at least once) the weights of all edges. 
$\omv$ is also a local approach, in the sense that prediction at each node
is only affected by the labels of adjacent nodes. 

\paragraph{Label Propagation.} Introduced by~\cite{ZGL03} and here abbreviated as $\labprop$.
%
This is a batch transductive learning method based on solving a (possibly sparse) 
linear system of equations which requires $\Theta(mn)$ time on an $n$-node
graph with $m$ edges. This bad scalability prevented us from carrying out
comparative experiments on larger graphs of $10^6$ or more nodes.
Note that \omv\ can be viewed as a fast approximation of \labprop.

In our experiments, we combined $\wta$ and $\gpa$ with spanning trees generated
in different ways (note that \omv\ and \labprop\ do not use spanning trees).

\paragraph{Random Spanning Tree} ($\rst$). 
Each spanning tree is taken with probability proportional to the product
of its edge weights ---see, e.g., \cite[Chapter 4]{LP09}.
In addition, we also tested \wta\ combined with \rst\ generated by ignoring
the edge weights (which were then restored before running \wta).
This second approach gives a prediction algorithm whose total expected running time, including the generation of the spanning tree, is $\Theta(n)$ on most graphs. We abbreviate this spanning tree as $\nwrst$ (non-weighted $\rst$).

\paragraph{Depth-first spanning tree} ($\dfst$). This spanning tree is created via
the following randomized depth-first visit: A root is selected at random, then each newly visited 
node is chosen with probability proportional to the weights of the edges connecting the
current vertex with the adjacent nodes that have not been visited yet.
This spanning tree is faster to generate than \rst, and can be viewed as an approximate
version of \rst.

\paragraph{Minimum Spanning Tree} ($\mst$). The spanning tree minimizing the sum of
the resistors of all edges. This is the tree whose Laplacian best approximates
the Laplacian of $G$ according to the trace norm criterion ---see, e.g., the paper of~\cite{HPR09}.

\paragraph{Shortest Path Spanning Tree} ($\spst$). \cite{HPR09} use the shortest path
tree because it has a small diameter (at most twice the diameter of $G$). This allows them to
better control the theoretical performance of $\gpa$. We generated several shortest path spanning
trees by choosing the root node at random, and then took the one with minimum diameter.

In order to check whether the information carried by the edge weight has predictive
value for a nearest neighbor rule like \wta, we also performed a test by ignoring the edge
weights during both the generation of the spanning tree and the running of \wta's nearest neighbor
rule.
This is essentially the algorithm analyzed by~\cite{HLP09}, and we denote it by
$\nwwta$ (non-weighted \wta). We combined \nwwta\ with weighted and unweighted spanning trees.
So, for instance, \nwwta+\rst\ runs a 1-NN rule (\nwwta) that does not
take edge weights into account (i.e., pretending that all weights are unitary) on a random 
spanning tree generated according to the actual edge weights. \nwwta+\nwrst\ runs \nwwta\ on a random spanning tree that also disregars edge weights.

Finally, in order to make the classifications based on \rst's more robust with respect to the variance associated with the random generation of the spanning tree, we also tested committees of \rst's. For example, K*\wta+\rst\ denotes the classifier obtained by drawing $K$ \rst's, running \wta\ on each one of them, and then aggegating the predictions of the $K$ resulting classifiers via a majority vote. For our experiments we chose $K=7,11,17$.

We ran our experiments on five real-world datasets:

\paragraph{RCV1.} The first 10,000 documents\footnote
{
Available at \texttt{trec.nist.gov/data/reuters/reuters.html}.
}
(in chronological order) of
Reuters Corpus Volume~1, with \textsc{tf-idf} preprocessing and Euclidean normalization.

\paragraph{USPS.} The USPS dataset\footnote
{ 
Available at \texttt{www-i6.informatik.rwth-aachen.de/$\tilde\ $keysers/usps.html}.
}
with features normalized into $[0,2]$.

\paragraph{KROGAN.} This is a high-throughput protein-protein interaction network for budding yeast. 
It has been used by~\cite{Kro06} and~\cite{PSGGK07}. 

\paragraph{COMBINED.} A second dataset from the work of~\cite{PSGGK07}.
It is a combination of three datasets: \cite{Gav02}'s, \cite{Ito01}'s, and~\cite{Uet00}'s.

\paragraph{WEBSPAM.} A large dataset (110,900 nodes and 1,836,136 edges) 
of inter-host links created for the\footnote
{
The dataset is available at \texttt{barcelona.research.yahoo.net/webspam/datasets/}.
We do not compare our results to those obtained in the challenge since we are only 
exploiting the graph (weighted) topology here, disregarding content features.
} 
Web Spam Challenge 2008~\cite{YWS07}. This is a weighted graph with binary labels 
and a pre-defined train/test split: 3,897 training nodes and 1,993 test nodes
(the remaining ones being unlabeled).

We created graphs from RCV1 and USPS with as many nodes as the total
number of examples $(\bx_i,y_i)$ in the datasets. That is, 10,000 nodes for RCV1 and 7291+2007 = 9298 for USPS.
Following previous experimental settings~\cite{ZGL03,BMN04}, the graphs were constructed using $k$-NN based on the
standard Euclidean distance $\norm{\bx_i-\bx_j}$ between node $i$ and node $j$. 
The weight $w_{i,j}$ was set to $w_{i,j} = \exp\bigl(-\norm{\bx_i-\bx_j}^2\big/\sigma_{i,j}^2\bigr)$,
if $j$ is one of the $k$ nearest neighbors of $i$, and 0 otherwise. 
To set $\sigma_{i,j}^2$, we first computed the average square distance 
between $i$ and its $k$ nearest neighbors (call it $\sigma_i^2$), then we computed
$\sigma_j^2$ in the same way, and finally set $\sigma_{i,j}^2 = \bigl(\sigma_i^2 + \sigma_j^2\bigr)\big/2$.
We generated two graphs for each dataset by running $k$-NN with $k = 10$ (RCV1-10 and USPS-10)
and $k = 100$ (RCV1-100 and USPS-100). The labels were set using the four most frequent categories
in RCV1  and all 10 categories in USPS.

In KROGAN and COMBINED we only considered the biggest connected components of both datasets, obtaining
2,169 nodes and 6,102 edges for KROGAN, and 2,871 nodes and 6,407 edges for COMBINED.
In these graphs, each node belongs to one or more classes, each class representing a 
gene function. We selected the set of functional labels at depth one in the 
FunCat classification scheme of the MIPS database~\cite{Rue04}, resulting
in seventeen classes per dataset. 

In order to associate binary classification tasks with the six non-binary datasets/graphs
(RCV1-10, RCV1-100, USPS-10, USPS-100, KROGAN, COMBINED) we binarized the corresponding
multiclass problems via a standard one-vs-rest scheme.
We thus obtained:
four binary classification tasks for RCV1-10 and RCV1-100, 
ten binary tasks for USPS-10 and USPS-100, seventeen binary tasks for both KROGAN and COMBINED.
For a given a binary task and dataset, we tried different proportions of training set and test set sizes.
In particular, we used training sets of size 5\%, 10\%, 25\% and 50\%. For any given size, the training sets were randomly selected.

We report error rates and F-measures on the test set, after macro-averaging over the binary tasks. 
The results are contained in Tables \ref{t:rcv1-k10}--\ref{t:spamresults} 
(Appendix \ref{app:tables}) and in Figures \ref{f:charts}--\ref{f:charts2}. 
Specifically, Tables \ref{t:rcv1-k10}--\ref{t:combined} contain results for all combinations 
of algorithms and train/test split for the first six datasets (i.e., all but WEBSPAM).

The WEBSPAM dataset is very large, and requires us a lot of computational resources 
in order to run experiments on this graph. Moreover, $\gpa$ has always shown inferior accuracy 
performance than the corresponding version of \wta\ (i.e., the one using the same kind of spanning tree) 
on all other datasets. Hence we decided not to go on any further with the refined implementation
of $\gpa$ on trees we mentioned above.
In Table~\ref{t:spamresults} we only report test error results on the four algorithms
\wta, \omv, \labprop, and \wta\ with a committee of seven (nonweighted) random spanning trees.

In our experimental setup we tried to control the sources of variance in the first six datasets as follows:
\begin{enumerate}
\item We first generated ten random permutations of the node indices for each one of the six graphs/datasets;
\item on each permutation we generated the training/test splits;
\item we computed $\mst$ and $\spst$ for each graph and made (for $\wta$, $\gpa$, $\omv$, and $\labprop$)
one run per permutation on each of the 4+4+10+10+17+17 = 62 binary problems, 
averaging results over permutations and splits;     
\item for each graph, we generated ten random instances for each one of $\rst$, $\nwrst$, $\dfst$, 
and then operated as in step~2, with a further averaging over the randomness in the tree generation. 
\end{enumerate}
Figure \ref{f:charts} extracts from Tables \ref{t:rcv1-k10}--\ref{t:combined} the error levels
of the best spanning tree performers,
and compared them to \omv\ and \labprop. For comparison purposes, we also displayed the error levels 
achieved by \wta\ operating on a committee of seventeen random spanning trees (see below).
Figure \ref{f:charts2} (left) contains the error level on WEBSPAM reported in Table~\ref{t:spamresults}.
Finally, Figure \ref{f:charts2} (right) is meant to emphasize the error rate differences between
\rst\ and \nwrst\ run with \wta.
\begin{figure}
\begin{center}
\begin{tabular}{c c}
&\\
\includegraphics[width=70mm]{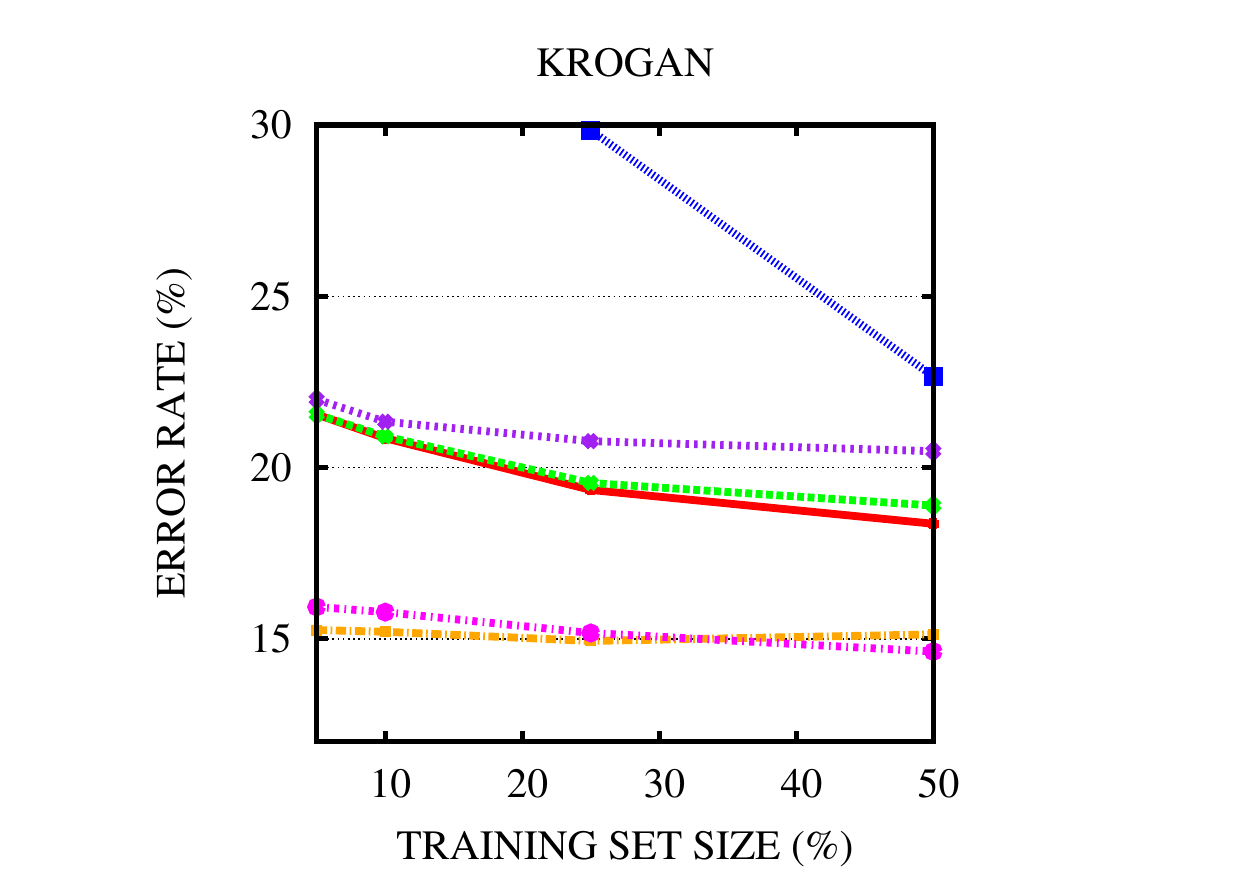}& \includegraphics[width=70mm]{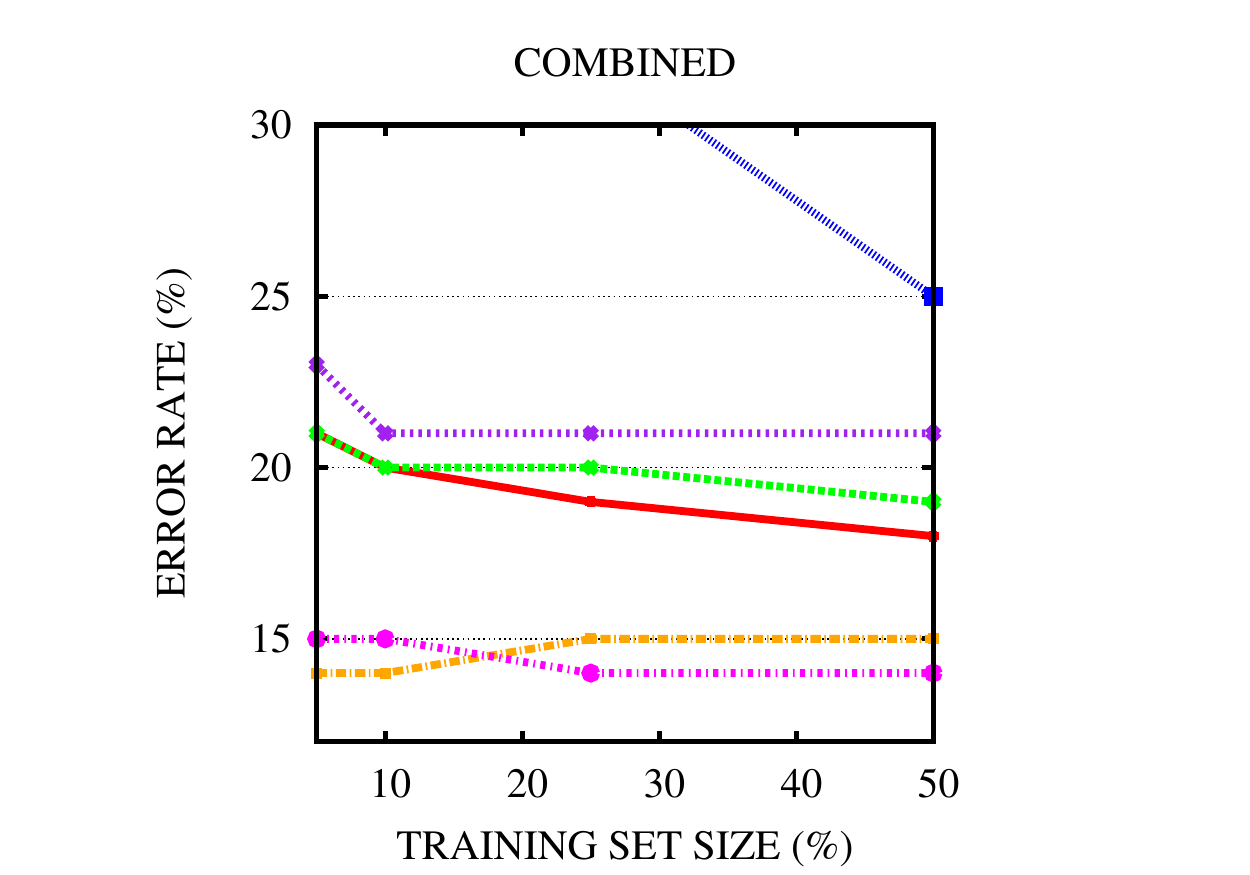}\\
\includegraphics[width=70mm]{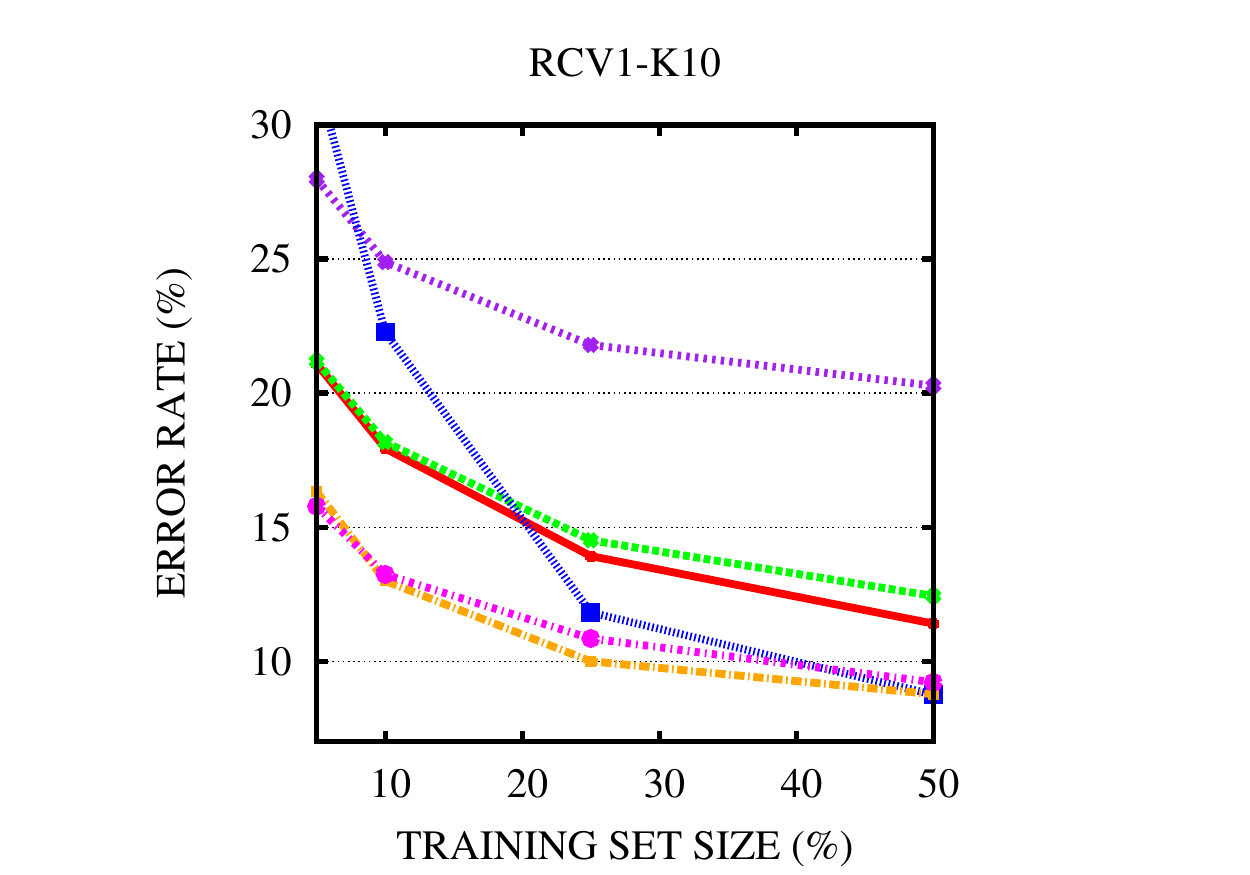}& \includegraphics[width=70mm]{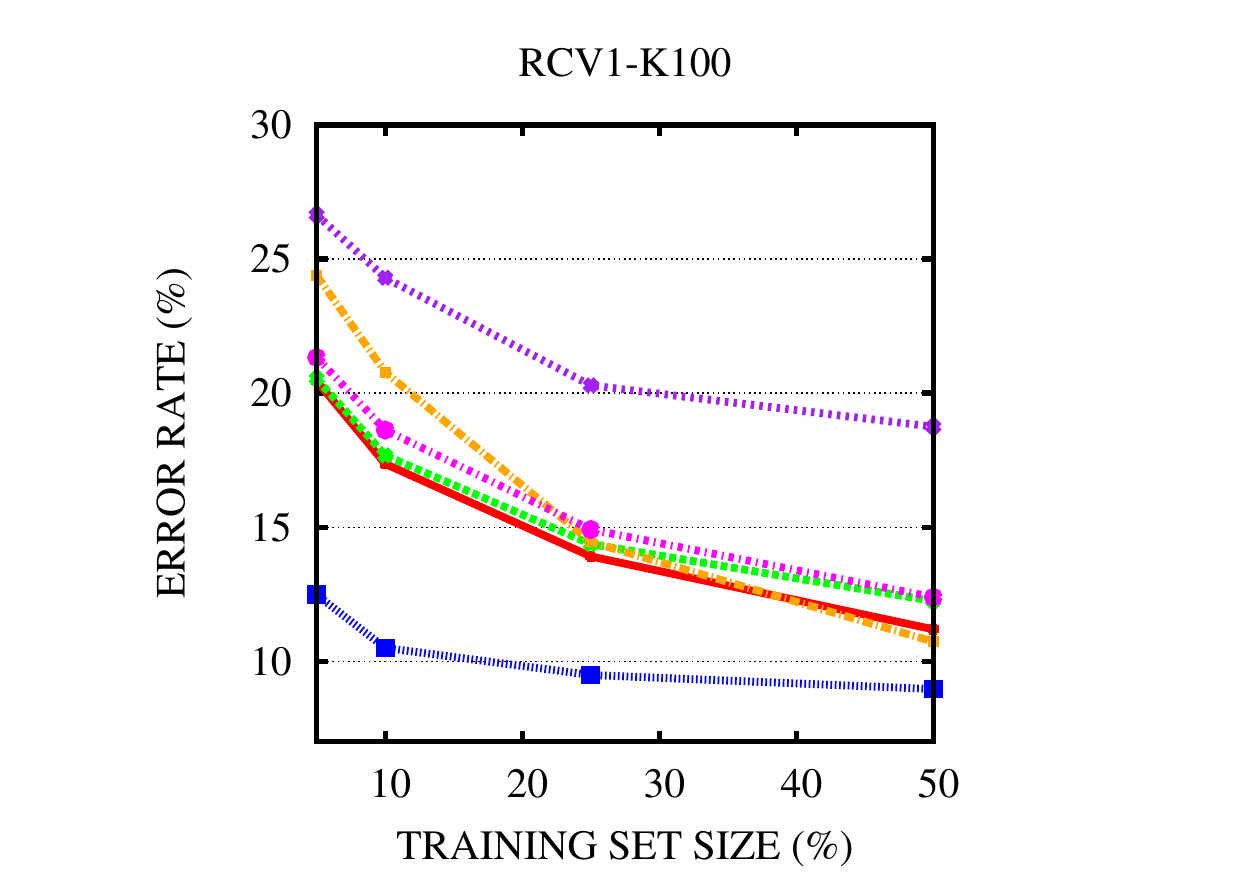}\\
\includegraphics[width=70mm]{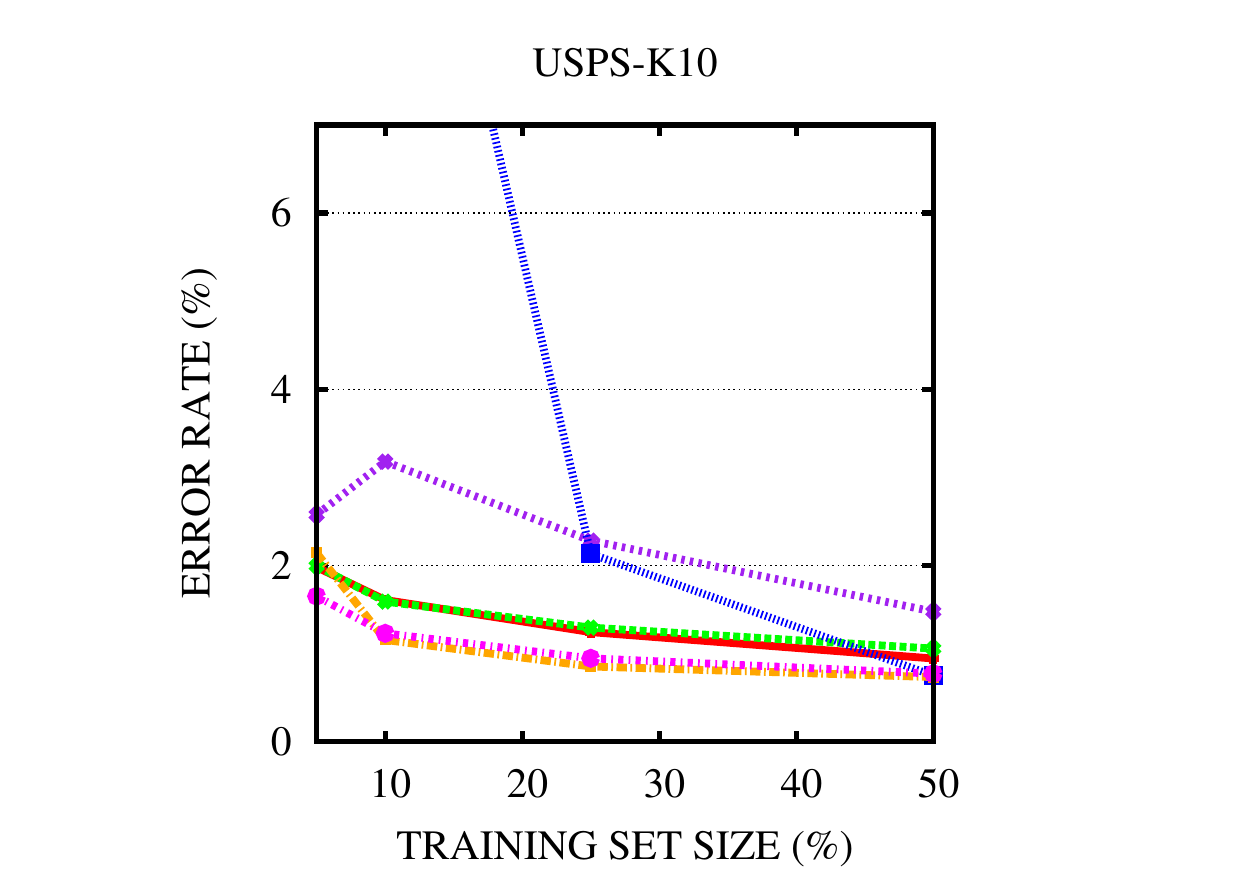}& \includegraphics[width=70mm]{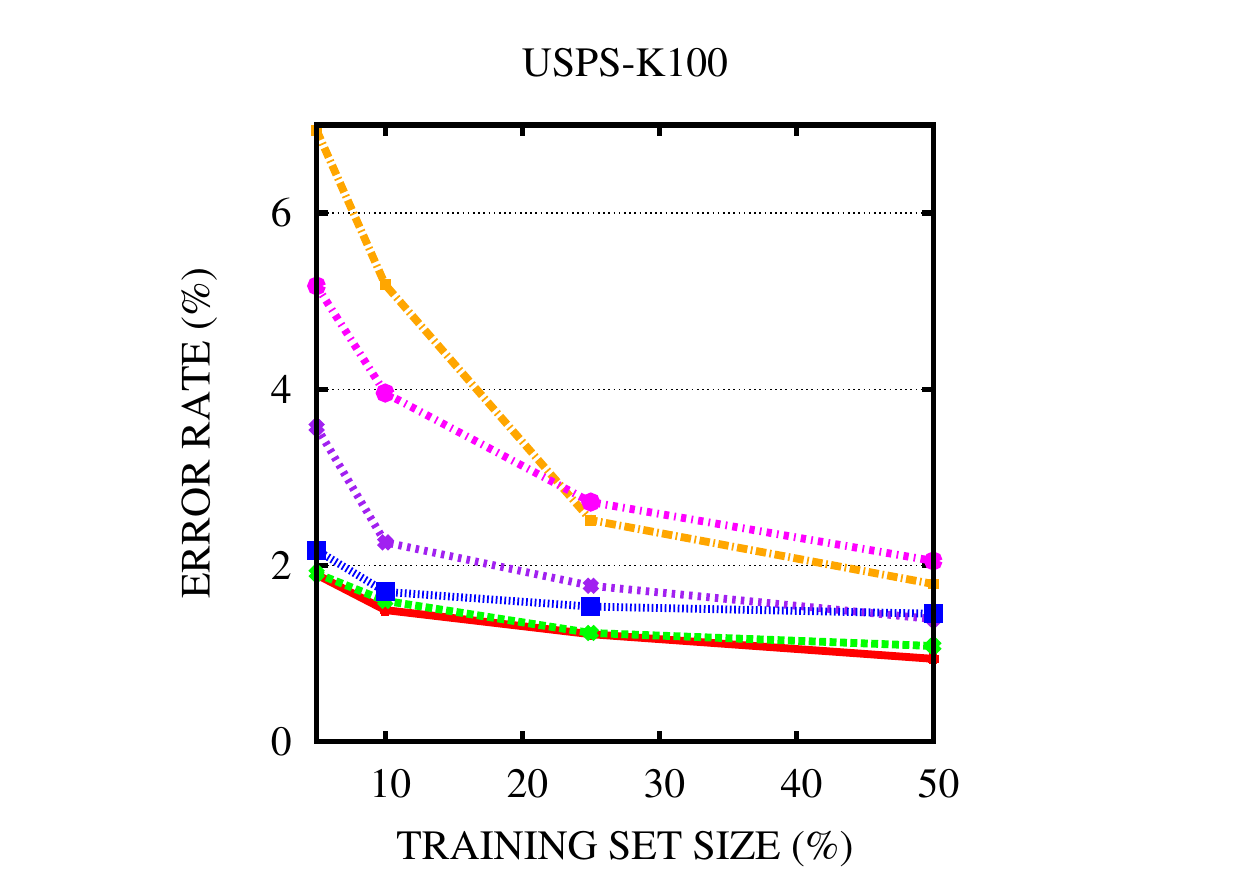}\\
\multicolumn{2}{c}{\includegraphics[width=70mm]{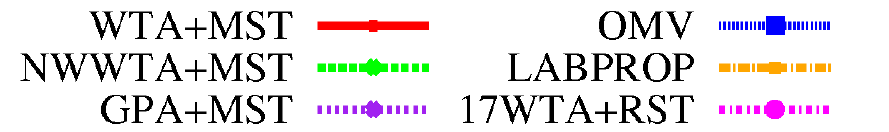}}
\end{tabular}
\end{center}
\caption{\label{f:charts}
Macroaveraged test error rates on the first six datasets as a function of the training set size.
The results are extracted from Tables~\ref{t:rcv1-k10}--\ref{t:combined} in Appendix B.
Only the best performing spanning tree (i.e., \mst) is shown for the algorithms that use spanning trees. These results are compared to \omv, \labprop, and 17*\wta+\rst.
}
\end{figure}
\begin{figure}
\begin{center}
\begin{tabular}{c c}
\includegraphics[width=70mm]{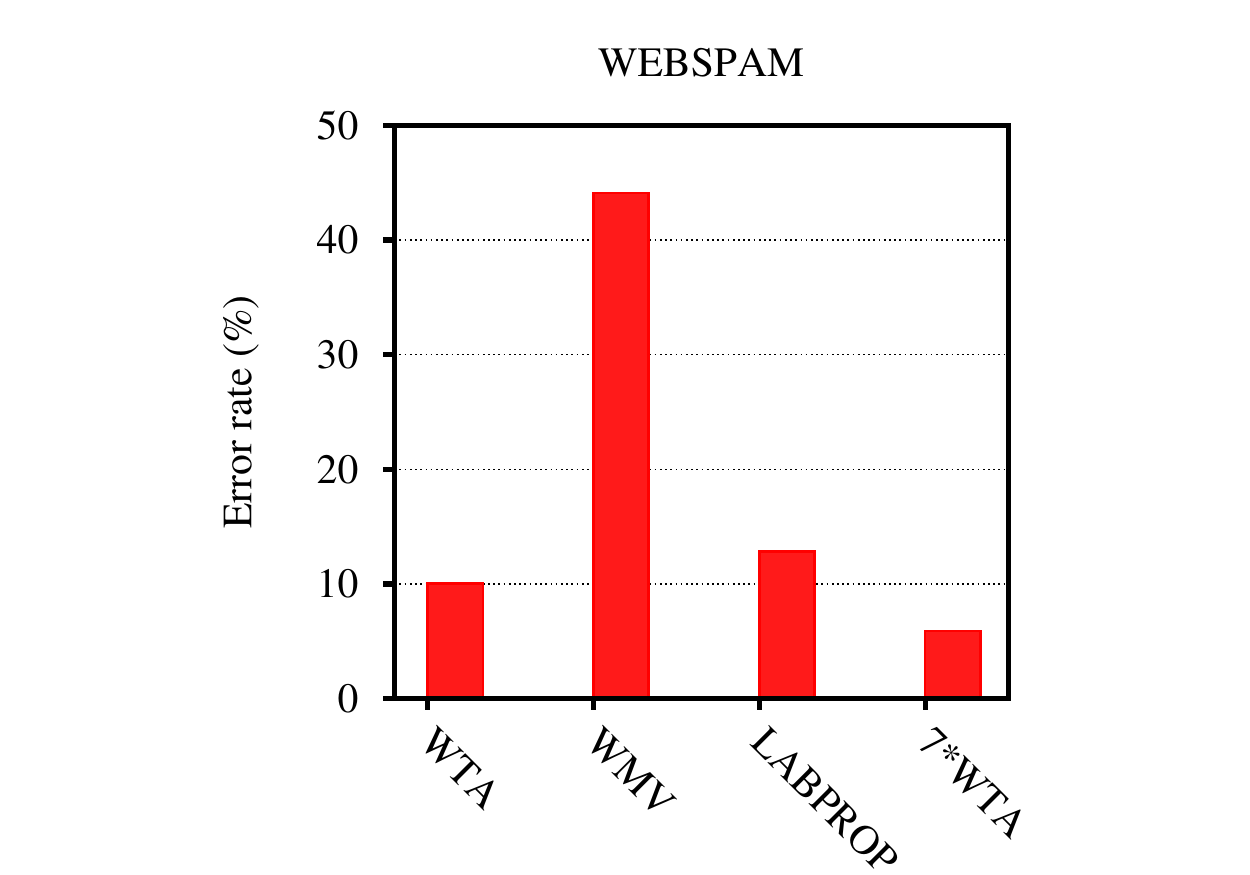}& \includegraphics[width=70mm]{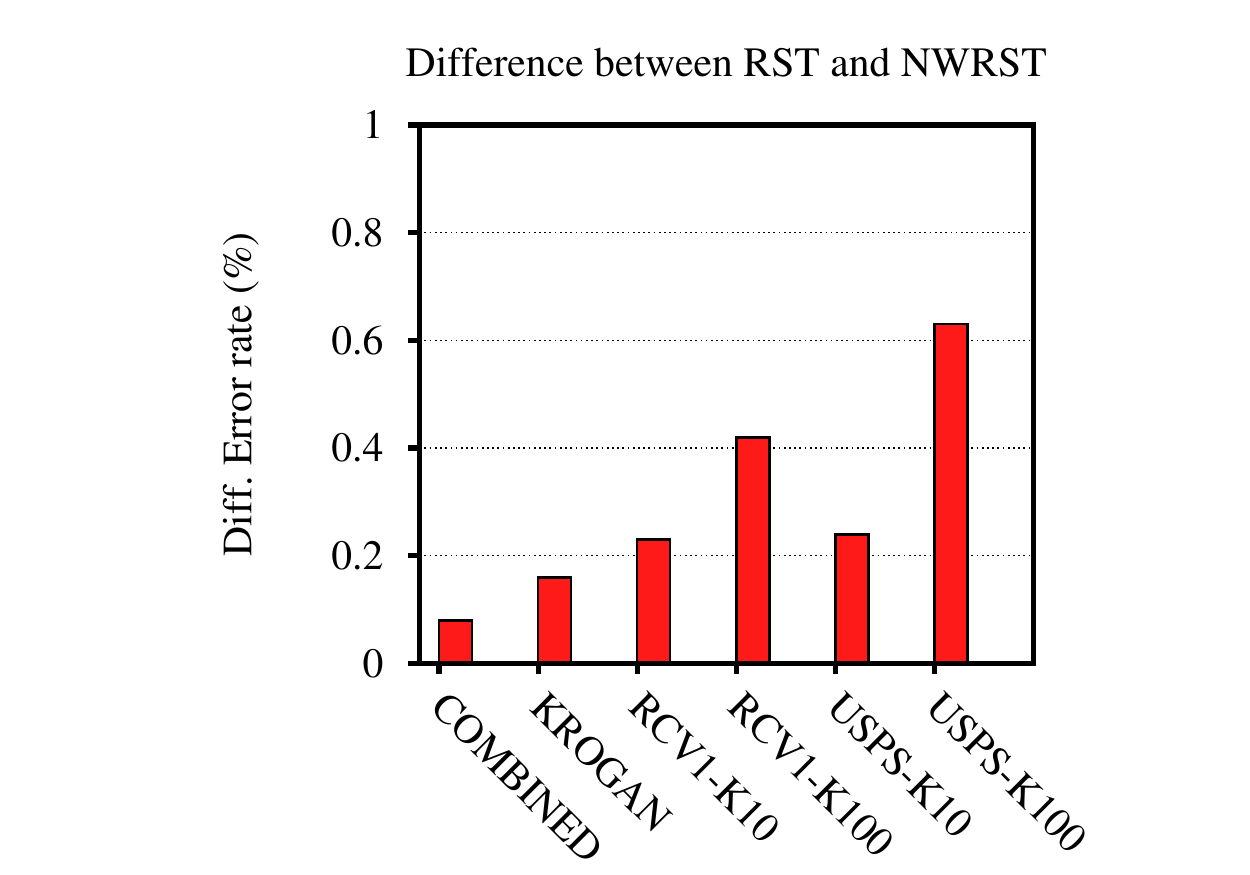}
\end{tabular}
\end{center}
\caption{\label{f:charts2}
\textbf{Left:} Error rate levels on WEBSPAM taken from Table \ref{t:spamresults} in Appendix \ref{app:tables}.
\textbf{Right:} Average error rate difference across datasets when using \wta+\nwrst\ rather than \wta+\rst.
}
\end{figure}
%


Several interesting observations and conclusions can be drawn from our experiments.
\begin{enumerate}
\item $\wta$ outperforms $\gpa$ on all datasets and with all spanning tree combinations. 
In particular, though we only reported aggregated results,
the same relative performance pattern among the two algorithms repeats systematically 
over all binary classification problems. In addition, $\wta$ runs significantly faster than $\gpa$,
requires less memory storage (linear in $n$, rather than quadratic), and is also fairly easy
to implement.
\item By comparing $\nwwta$ to $\wta$, we see that the edge weight information in the nearest neighbor 
rule increases accuracy, though only by a small amount.
\item $\omv$ is a fast and accurate approximation to $\labprop$ when either the graph is dense (RCV1-100, and
USPS-100) or the training set is comparatively large (25\%--50\%), although neither of the two situations often 
occurs in real-world applications.
\item The best performing spanning tree for both $\wta$ and $\gpa$ is $\mst$. 
This might be explained by the fact that $\mst$ tends to select light $\phi$-edges of the 
original graph.
\item $\nwrst$ and \dfst\ are fast approximations to $\rst$. 
Though the use of $\nwrst$ and \dfst\ does not provide theoretical performance guarantees as for $\rst$, 
in our experiments they do actually perform comparably. Hence, in practice, \nwrst\ and \dfst\ might 
be viewed as fast and practical ways to generate spanning trees for \wta.
\item The prediction performance of $\wta$+$\mst$ is sometimes slightly inferior to \labprop's.
However, it should be stressed that \labprop\ takes time $\Theta(mn)$, where $m$ is the number of edges, 
whereas a single sweep of $\wta$+$\mst$ over the graph just takes
time ${\mathcal{O}}(m\log n)$.\footnote
{
The $\mst$ of a graph $G = (V,E)$ can be computed in time ${\mathcal{O}}(|E|\log |V|)$. Slightly faster
implementations do actually exist which rely on Fibonacci heaps.
}
Committees of spanning trees are a simple way to make \wta\ approach, and sometimes surpass, the performance of \labprop.
One can see that on sparse graphs using committees gives a good performances improvement. 
In particular, committees of $\wta$ can reach the same performances of $\labprop$ while adding just a constant 
factor to their (linear) time complexity.
%
%
\end{enumerate}

\section{Conclusions and Open Questions}\label{s:concl}
We introduced and analyzed \wta, a randomized online prediction algorithm for weighted graph
prediction. The algorithm uses random spanning trees and has nearly optimal performance guarantees in terms of expected prediction accuracy. The expected running time of \wta\ is optimal when the random spanning tree is drawn ignoring edge weigths. Thanks to its linearization phase, the algorithm is also provably robust to label noise.

Our experimental evaluation shows that \wta\ outperforms other previously
proposed online predictors.
Moreover, when combined with an aggregation of random spanning trees, \wta\
also tends to beat standard batch predictors, such as label propagation.
These features make \wta\ (and its combinations) suitable to large scale applications. 

There are two main directions in which this work can improved. First, previous analyses~\cite{CGV09b}
reveal that \wta's analysis is loose, at least when the input graph is an unweighted tree with small diameter.
This is the main source of the $\Omega(\ln|V|)$ slack between \wta\ upper bound and the general lower bound
of Theorem~\ref{th:lower}. So we ask whether, at least in certain cases, this slack could be reduced.
Second, in our analysis we express our upper and lower bounds in terms
of the cutsize. One may object that a more natural quantity for our setting is the weighted cutsize, as this
better reflects the assumption that $\phi$-edges tend to be light, a natural notion of bias for weighted graphs.
In more generality, we ask what are other criteria that make a notion of bias better than another one. For
example, we may prefer a bias which is robust to small perturbations of the problem instance. In this sense
$\Phi_G^*$, the cutsize robust to label perturbation introduced in Section~\ref{s:robust}, is a better bias
than $\bE\,\Phi_T$. We thus ask whether there is a notion of bias, more natural and robust than $\bE\,\Phi_T$, which
captures as tightly as possible the optimal number of online mistakes on general weighted graphs. A partial
answer to this question is provided by the recent work of~\cite{VCGZ12}.
It would also be nice to tie this machinery with recent results in the active node classification setting
on trees contained in \cite{CGVZ10b}.

\ \\
\noindent{\bf Acknowledgments }
This work was supported in part by
Google Inc.\ through a Google Research Award, and by
the PASCAL2 Network of Excellence under EC grant~216886.
This publication only reflects the authors’ views.

\appendix
\section*{Appendix A}
\label{app:theorem}
\newtheorem{fact}{Fact}


This appendix contains the proofs of Lemma~\ref{l:ub-L-to-T},
Theorem~\ref{t:ub-tree}, Theorem~\ref{th:graph}, Corollary~\ref{cor:upper},
Theorem~\ref{t:robust}, Theorem~\ref{t:ub-tree-r}, Corollary~\ref{c:ub-tree-r-poly},
Theorem~\ref{th:graph_r}, and Corollary~\ref{c:graph_pol_r}.
Notation and references are as in the main text.
We start by proving Lemma~\ref{l:ub-L-to-T}.

\smallskip\noindent
\begin{proof}[Lemma \ref{l:ub-L-to-T}]
Let a {\em cluster} be any maximal sub-line of $L$ whose edges
are all $\phi$-free. Then $L$ contains exactly $\Phi_L(\by) + 1$ clusters,
which we number consecutively, starting from one of the two terminal nodes.
Consider the $k$-th cluster $c_k$. 
Let $v_0$ be the first node of $c_k$ whose label
is predicted by \wta. After $y_{v_0}$ is revealed, the cluster splits into 
two edge-disjoint sub-lines
$c'_k$ and $c''_k$, both having $v_0$ as terminal node.\footnote
{
With no loss of generality, we assume that neither of the two sub-lines is empty,
so that $v_0$ is not a terminal node of $c_k$.
}
Let $v'_k$ and $v''_k$ be the closest nodes to $v_0$ such that (i) 
$y_{v'_k} = y_{v''_k} \neq y_{v_0}$ and 
(ii) $v'_k$ is adjacent to a terminal node of $c'_k$, and 
$v''_k$ is adjacent to a terminal node of $c''_k$.
The nearest neighbor prediction rule of
\wta\ guarantees that the first mistake made on $c'_k$ (respectively, $c''_k$) 
must occur on a node $v_1$ such that $d(v_0, v_1) \ge d(v_1, v'_k)$
(respectively, $d(v_0, v_1) \ge d(v_1, v''_k)$). 
By iterating this argument for the subsequent mistakes we
see that the total number of mistakes made on cluster $c_k$ is bounded by
\[
1+\left\lfloor \log_2 \frac{R'_k+(w'_k)^{-1}}{(w'_k)^{-1}} \right\rfloor
 + \left\lfloor \log_2 \frac{R''_k+(w''_k)^{-1}}{(w''_k)^{-1}}\right\rfloor
\]
where $R'_k$ is the resistance diameter of sub-line $c'_k$, and $w'_k$
is the weight of the $\phi$-edge between $v'_k$ and the terminal
node of $c'_k$ closest to it ($R''_k$ and $w''_k$ are defined similarly).
Hence, summing the above displayed expression over clusters $k = 1, \ldots, \Phi_L(\by) + 1$ 
we obtain
\begin{align*}
    m_L
&\bigoheq
    \Phi_L(\by)
    + \sum_{k} \Bigl(\log \bigl(1+R'_k w'_k\bigr)
    + \log \bigl(1+R'_k w''_k\bigr)\Bigr)
\\ &\bigoheq
    \Phi_L(\by) \left(1 + \log\left(1+\frac{1}{\Phi_L(\by)}\sum_{k}R'_k w'_k
    \right)\right.
 + \left. \log \left(1 + \frac{1}{\Phi_L(\by)}\sum_{k}R''_k w''_k
    \right)\right)
\\ &\bigoheq
    \Phi_L(\by) \left(1 + \log \left(1+\frac{R^W_L\Phi_L^W(\by)}{\Phi_L(\by)}\right) \right)
\end{align*}
where in the second step we used Jensen's inequality and in the last 
one the fact that $\sum_k (R'_k + R''_k) = R^W_L$ and
$\max_k w'_k \bigoheq \Phi_L^W(\by)$, $\max_k w''_k \bigoheq \Phi_L^W(\by)$.
This proves the lemma in the case $E' \equiv \emptyset$.

In order to conclude the proof, observe that if we take any semi-cluster
$c'_k$ (obtained, as before, by splitting cluster $c_k$, being $v_0 \in c_k$ the 
first node whose label is predicted by \wta), and pretend
to split it into two sub-clusters connected by a $\phi$-free edge, we could repeat
the previous dichotomic argument almost verbatim on the two sub-clusters at the cost of
adding an extra mistake. 
%
We now make this intuitive argument more precise. Let $(i,j)$ be a $\phi$-free edge belonging to semi-cluster $c'_k$, and suppose
without loss of generality that $i$ is closer to $v_0$ than to $j$. If we remove edge
$(i,j)$ then $c'_k$ splits into two subclusters: $c'_k(v_0)$ and $c'_k(j)$, containing node
$v_0$ and $j$, respectively (see Figure~\ref{f:plusk}). Let  $m_{c'_k}$, $m_{c'_k(v_0)}$ and $m_{c'_k(j)}$ be  the number of mistakes
made on $c'_k$, $c'_k(v_0)$ and $c'_k(j)$, respectively.
We clearly have $m_{c'_k} = m_{c'_k(v_0)} + m_{c'_k(j)}$. 

\begin{figure}[h!]
  \centering
\figscale{\imgdir/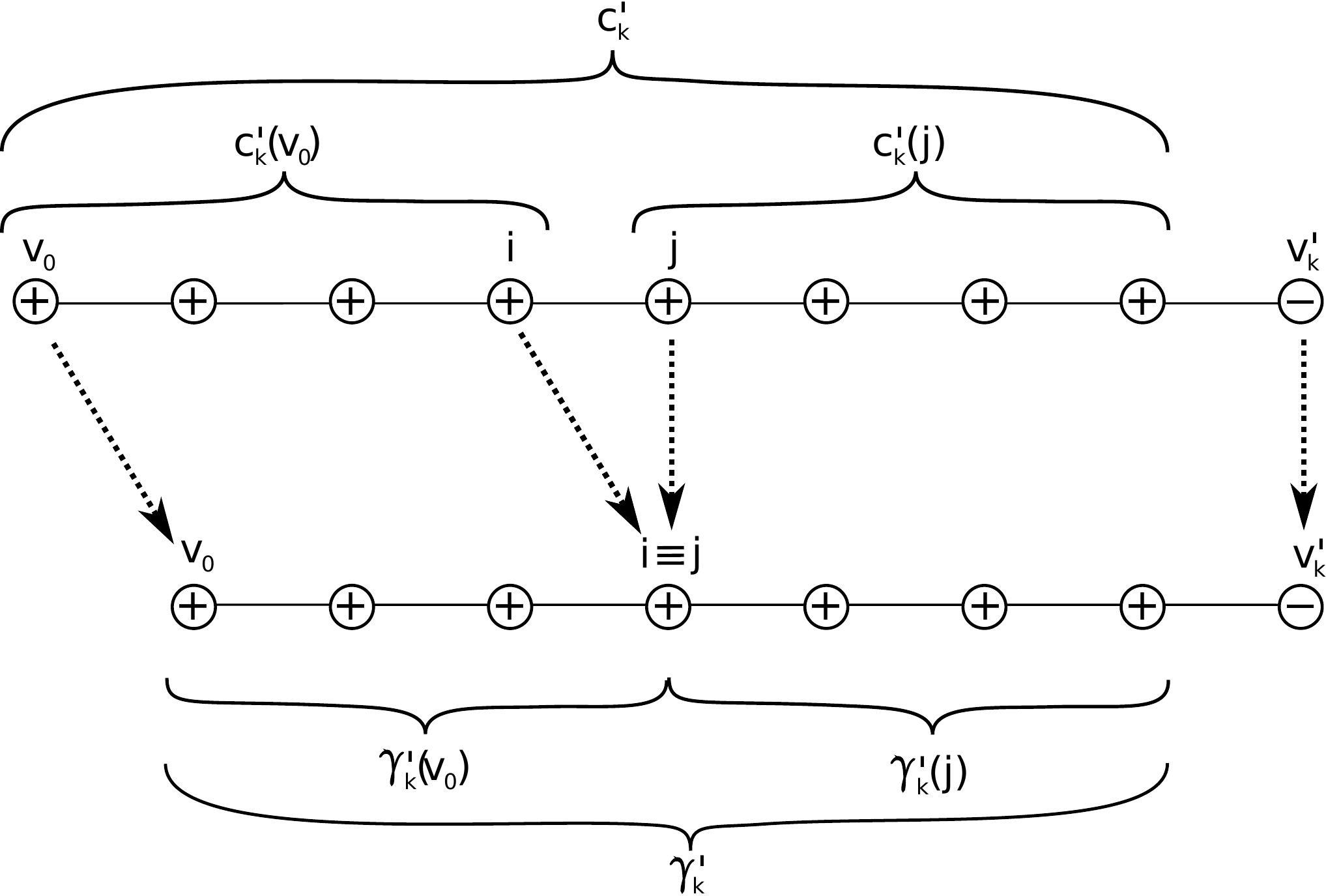}{0.5}
\caption{\label{f:plusk} We illustrate the way we bound the number of mistakes
on semi-cluster $c'_k$ by dropping the resistance contribution of any (possibly very light) 
edge $(i,j)$, at the cost of increasing the mistake bound on $c'_k$ by $1$. 
The removal of $(i,j)$ makes $c'_k$ split into subclusters $c'_k(v_0)$ and $c'_k(j)$.  
We can then 
drop edge $(i,j)$ by making node $i$ coincide with node $j$. The resulting 
semi-cluster is denoted $\gamma'_k$. 
This shortened version of $c'_k$ can be viewed as split into  sub-cluster $\gamma'_k(v_0)$ and subcluster
$\gamma'_k(j)$, corresponding to $c'_k(v_0)$ and $c'_k(j)$, respectively. 
Now, the number of mistakes made on $c'_k(v_0)$ and $c'_k(j)$ can be bounded
by those made on $\gamma'_k(v_0)$ and $\gamma'_k(j)$. Hence, we can bound
the mistakes on $c'_k$ through the ones made on $\gamma'_k$, with the addition of a single 
mistake, rather than two, due to the double node $i \equiv j$ of $\gamma'_k$.
}
\end{figure}

Let now $\gamma'_k$ be the semi-cluster obtained from $c'_k$ by contracting edge 
$(i,j)$ so as to make $i$ coincide with $j$ (we sometimes write $i \equiv j$). 
Cluster $\gamma'_k$ can be split into 
two parts which overlap only at node $i \equiv j$: $\gamma'_k(v_0)$, with terminal nodes $v_0$
and $i$ (coinciding with node $j$), and $\gamma'_k(j)$. 
In a similar fashion, let $m_{\gamma'_k}$, $m_{\gamma'_k(v_0)}$,
and $m_{\gamma'_k(j)}$ be the number of mistakes
made on $\gamma'_k$, $\gamma'_k(v_0)$ and $\gamma'_k(j)$, respectively.
We have $m_{\gamma'_k} = m_{\gamma'_k(v_0)} + m_{\gamma'_k(j)} -1$, where the $-1$ 
takes into account that $\gamma'_k(v_0)$ and $\gamma'_k(j)$ 
overlap at node $i \equiv j$.

Observing now that, for each node $v$ belonging to $c'_k(v_0)$ (and $\gamma'_k(v_0))$, 
the distance $d(v,v'_k)$ is smaller on $\gamma_k$ than on $c'_k$,
we can apply the abovementioned dichotomic argument 
to bound the mistakes made on $c'_k$, obtaining $m_{\gamma'_k(v_0)} \le m_{c'_k(v_0)}$.
Since $m_{c'_k(j)} = m_{\gamma'_k(j)}$, we can finally write
$
m_{c'_k} = m_{c'_k(v_0)} + m_{c'_k(j)} \le  m_{\gamma'_k(v_0)} +  m_{\gamma'_k(j)}    = m_{\gamma'_k} + 1
$.
Iterating this argument for all edges in $E'$ concludes the proof.
\end{proof}
In view of proving Theorem~\ref{t:ub-tree}, we now prove the following two lemmas.
\begin{lemma}\label{line_edges}
Given any tree $T$, let $E(T)$ be the edge set of $T$, and let $E(L')$ and $E(L)$ be 
the edge sets of line graphs $L'$ and $L$ obtained via \wta's tree linearization of $T$.
Then the following holds.
\begin{enumerate}
\item 
There exists a partition $\scP_{L'}$ of $E(L')$ in pairs and a bijective mapping $\mu_{L'}: \scP_{L'} \to E(T)$ such that the weight of both edges in each pair $S' \in \scP_{L'}$ is equal to the weight of the edge $\mu_{L'}(S')$.
\item
There exists a partition $\scP_{L}$ of $E(L)$ in sets $S$ such that $|S|\le 2$, and there exists an injective mapping $\mu_{L} : \scP_{L} \to E(T)$ such that the weight of the edges in each pair $S \in \scP_{L}$ is equal to the weight of the edge $\mu_{L}(S)$.
\end{enumerate}
\end{lemma}
\begin{proof}
We start by defining the bijective mapping $\mu_{L'}: \scP_{L'} \to E(T)$. Since each edge $(i,j)$ of $T$ is traversed exactly twice
in the depth-first visit that generates $L'$,\footnote
{
For the sake of simplicity, we are assuming here that the depth-first visit of $T$ 
terminates by backtracking over all nodes on the path between the last node visited
in a forward step and the root.
} 
once in a forward step and once in a backward step, we partition $E(L')$ in pairs $S'$ such that $\mu_{L'}(S')=(i,j)$ if and only if $S'$ contains the pair of distinct edges created in $L'$ by the two traversals of $(i,j)$. By construction, the edges in each pair $S'$ have weight equal to $\mu_{L'}(S')$. Moreover, this mapping is clearly bijective, since any edge of $L'$ is created by a single traversal of an edge in $T$. The second mapping $\mu_{L}: \scP(L) \to E(T)$ is created as follows. $\scP_L$ is created from $\scP_{L'}$ by removing from each $S'\in\scP_{L'}$ the edges that are eliminated when $L'$ is transformed into $L$. Note that we have $\bigl|\scP_L\bigr|\le\bigl|\scP_{L'}\bigr|$ and for any $S\in\scP_L$ there is a unique $S'\in\scP_{L'}$ such that $S \ss S'$. Now, for each $S\in\scP_L$ let $\mu_L(S) = \mu_{L'}(S')$, where $S'$ is such that $S \ss S'$. Since $\mu_{L'}$ is bijective, $\mu_L$ is injective. Moreover, since the edges in $S'$ have the same weight as the edge $\mu_{L'}(S')$, the same property holds for $\mu_L$.
%
%
\end{proof}
%
%
\begin{lemma}
\label{lemma_tree_line}
Let $(T,\by)$ be a labeled tree, let $(L,\by)$ be the linearization of $T$, 
and let $L'$ be the line graph with duplicates (as described above). 
Then the following holds.\footnote
{
Item 2 in this lemma is essentially contained in the paper by~\cite{HLP09}.
}
\begin{enumerate}
\item $\Phi_L^W(\by) \le \Phi_{L'}^W(\by) \le 2 \Phi_T^W(\by)$;
\item $\Phi_L(\by) \le \Phi_{L'}(\by) \le 2\Phi_T(\by)$.
\end{enumerate}
\end{lemma}
\begin{proof}
From Lemma~\ref{line_edges} (part~1) we know that $L'$ contains a duplicated edge for each edge of $T$. This immediately implies  $\Phi_{L'}(\by) \le 2 \Phi_T(\by)$ and $\Phi^W_{L'}(\by) \le 2\Phi^W_T(\by)$.

To prove the remaining inequalities, note that from the description of \wta\ in Section~\ref{s:alg} (step~{3}), we see that when $L'$ is transformed into $L$ the pair of edges $(j',j)$ and $(j,j'')$ of $L'$, which are
incident to a duplicate node $j$, gets replaced in $L$ (together with $j$) by a single edge $(j',j'')$.
Now each such edge $(j',j'')$ cannot be a
$\phi$-edge in $L$ unless either $(j,j')$ or $(j,j'')$ is a $\phi$-edge in $L'$, and this establishes $\Phi_L(\by) \le \Phi_{L'}(\by)$. Finally, if $(j',j'')$ is a $\phi$-edge in $L$, then its weight is not larger than the
weight of the associated $\phi$-edge in $L'$ (step~{3} of $\wta$), and this establishes $\Phi_L^W(\by) \le \Phi_{L'}^W(\by)$.
\end{proof}
Recall that, given a labeled graph $G = (V,E)$ and any $\phi$-free edge subset $E' \subset E \setminus E^{\phi}$, the quantity $R^W_{G}(\neg E')$ is the sum of the resistors of all $\phi$-free edges in $E \setminus (E^{\phi} \cup E')$.
\begin{lemma}\label{l:ub_r}
If \wta\ is run on a weighted line graph $(L,\by)$ obtained
through the linearization of a given labeled tree $(T,\by)$ with edge set $E$, then the total number
$m_T$ of mistakes satisfies
\[
    m_T
\bigoheq
    \Phi_L(\by)
    \left(1 + \log_2\left(1+\frac{R^W_T(\neg E') \ \Phi^W_L(\by)}{\Phi_L(\by)}\right)
    \right)+\Phi_T(\by)+|E'|~,
\]
where $E'$ is an arbitrary subset of $E \setminus E^{\phi}$.
\end{lemma}
%
%
\begin{proof}
Lemma~\ref{line_edges} (Part~2), exhibits an injective mapping $\mu_L : \scP \to E$, where $\scP$ is a partition of the edge set $E(L)$ of $L$, such that every $S\in\scP$ satisfies $|S|\le 2$. Hence, we have $|E'(L)| \le 2|E'|$, where $E'(L)$ is the union of the pre-images of edges in $E'$ according to $\mu_L$ ---note that some edge in $E'$ might not have a pre-image in $E(L)$. By the same argument, we also establish $|E_0(L)| \le 2\Phi_T$, where $E_0(L)$ is the set of $\phi$-free edges of $L$ that belong to elements $S$ of the partition $\scP_L$ such that $\mu_L(S) \in E^{\phi}$.

Since the edges of $L$ that are neither in $E_0(L)$ nor in $E'(L)$ are partitioned by $\scP_L$
in edge sets having cardinality at most two, which in turn can be injectively mapped
via $\mu_L$ to $E \setminus (E^{\phi} \cup E')$, we have
$
    R^W_L\Bigl(\neg \bigl(E'(L) \cup E_0(L)\bigr)\Bigr) \le 2 R^W_T(\neg E')~.
$
Finally, we use $|E'(L)| \le 2|E'|$ and $|E_0(L)| \le 2\Phi_{T}(\by)$  (which we just established) and apply Lemma~\ref{l:ub-L-to-T} with $E' \equiv E'(L) \cup E_0(L)$. This
concludes the proof.
\end{proof}
\begin{proof}[of Theorem \ref{t:ub-tree}]
We use Lemma~\ref{lemma_tree_line} to establish $\Phi_L(\by) \le 2\Phi_T(\by)$ and $\Phi_L^W(\by) \le 2\Phi_T^W(\by)$.
We then conclude with an application of Lemma~\ref{l:ub_r}.
\end{proof}
%
%
%
%
\begin{lemma}\label{l:graph_lemma}
If \wta\ is run on a weighted line graph $(L,\by)$ obtained
through the linearization of random spanning tree $T$ of a labeled weighted graph
$(G,\by)$, then the total number $m_G$ of mistakes satisfies%
\[
    \bE\,m_G
\bigoheq
    \bE\bigl[\Phi_L(\by)\bigr]\left(1 + \log \left(1 + \wmax^{\phi} \bE\bigl[R^W_T\bigr] \right) + \bE\bigl[\Phi_T(\by)\bigr]\right)~,
\]
where $\wmax^{\phi} = \max_{(i,j) \in E^{\phi}} w_{i,j}$.
\end{lemma}
\begin{proof}
Using Lemma~\ref{l:ub_r} with $E' \equiv \emptyset$ we can write
\begin{align*}
    \bE\,m_G
&\bigoheq
    \bE\Biggl[\Phi_L(\by)\left(1 + \log \left(1+\frac{R^W_T \Phi^W_L(\by)}{\Phi_L(\by)} \right) \right)+\Phi_T\Biggr]
\\ &\bigoheq
    \bE\Bigl[\Phi_L(\by)\Bigl(1 + \log \left(1+R^W_T \wmax^{\phi} \Bigr) \right)+\Phi_T\Bigr] 
\\ &\bigoheq
    \bE\bigl[\Phi_L(\by)\bigr]\Bigl(1 + \log \left(1+\bE\bigl[R^W_T\bigr] \wmax^{\phi} \Bigr) \right) + \bE\bigl[\Phi_T(\by)\bigr]~,
\end{align*}
where the second equality follows from the fact that $\Phi^W_L(\by) \le \Phi_L(\by)\wmax^{\phi}$, which in turn follows from Lemma~\ref{line_edges},
and the third one follows from Jensen's inequality applied to the concave function
$(x,y) \mapsto x\Bigl(1 + \log \left(1+y\,\wmax^{\phi} \Bigr) \right)$ for $x,y \geq 0$.
\end{proof}

\medskip

\begin{proof}[Theorem \ref{th:graph}]
We apply Lemma~\ref{l:graph_lemma} and then Lemma~\ref{lemma_tree_line} to get $\Phi_L(\by) \le 2 \Phi_T(\by)$.
\end{proof}
\begin{proof}[Corollary \ref{cor:upper}]
%
Let $f > \mathrm{poly(n)}$ denote a function growing faster than
any polynomial in $n$.
Choose a polynomially connected graph $G$ and a labeling $\by$.
For the sake of contradiction, assume that $\wta$ makes more than 
$\scO(\bE \bigl[\Phi_T(\by)\bigr] \log n)$
mistakes on $(G,\by)$. Then Theorem\ \ref{th:graph} implies
$w^{\phi}_{\max}\bE\bigl[R^W_T\bigr] > \mathrm{poly}(n)$.
Since $\bE\bigl[R^W_T\bigr] = \sum_{(i,j) \in E \setminus E^{\phi}} r^W_{i,j}$,
we have that $w^{\phi}_{\max} \max_{(i,j) \in E \setminus E^{\phi}}r^W_{i,j} > \mathrm{poly}(n)$.
Together with the assumption of polynomial connectivity for $G$, this implies
$w^{\phi}_{\max} r^W_{i,j} > \mathrm{poly}(n)$ for all $\phi$-free edges $(i,j)$.
By definition of effective resistance, $w_{i,j} r^W_{i,j} \le 1$ for all $(i,j) \in E$.
This gives $w^{\phi}_{\max}/w_{i,j} > \mathrm{poly}(n)$ for all $\phi$-free edges $(i,j)$,
which in turn implies
\[
\frac{\sum_{(i,j) \in E^{\phi}} w_{i,j}}{\sum_{(i,j) \in E \setminus E^{\phi}} w_{i,j}} > \mathrm{poly}(n)~.
\]
As this contradicts our hypothesis, the proof is concluded.
\end{proof}
\begin{proof}[Theorem \ref{t:robust}]
We only prove the first part of the theorem. The proof of the second part corresponds to the special case when all weights are equal to $1$.

Let $\Delta(\by,\by') \subseteq V$ be the set of nodes $i$
such that $y_i \neq y_i'$. We therefore have $\delta(\by,\by') = |\Delta(\by,\by')|$.
Since in a line graph each node is adjacent to at most two other nodes,
the label flip of any node $j \in \Delta(\by,\by')$
can cause an increase of the weighted cutsize
of $L$ by at most $w_{i',j} + w_{j,i''}$,
where $i'$ and $i''$ are the two nodes adjacent to $j$ in $L$.\footnote
{
In the special case when $j$ is terminal node we can set $w_{j,i''} = 0$.
}
Hence, flipping the labels of all nodes in $\Delta(\by,\by')$,
we have that the total cutsize increase is bounded by the sum
of the weights of the $2 \delta(\by,\by')$
heaviest edges in $L$, which implies 
\[
\Phi^W_L(\by) \le \Phi^W_L(\by') + \zeta_L\bigl(2\delta(\by,\by')\bigr)~.
\]
By Lemma~\ref{lemma_tree_line}, $\Phi^W_L(\bu) \le 2\Phi^W_T(\bu)$.
Moreover, Lemma~\ref{line_edges} gives an injective mapping $\mu_L : \scP_L \to E$ ($E$ is the edge set of $T$)
such that the elements of $\scP$ have cardinality at most two, and the weight of each edge $\mu_L(S)$ is the same as
the weights of the edges in $S$. Hence, the total weight of the $2\delta(\by,\by')$ heaviest edges in $L$ is at
most twice the total weight of the $\delta(\by,\by')$ heaviest edges in $T$.
Therefore $\zeta_L\bigl(2\delta(\by,\by')\bigr) \le 2\zeta_T\bigl(\delta(\by,\by')\bigr)$.
Hence, we have obtained
\[
	\Phi^W_L(\by) \le 2\Phi^W_T(\by')+ 2\zeta_T\bigl(\delta(\by,\by')\bigr)~,
\]
concluding the proof.
%
\end{proof}
\begin{proof}[Theorem \ref{t:ub-tree-r}]
We use Theorem~\ref{t:robust} to bound $\Phi_L(\by)$ and $\Phi_L^W(\by)$ 
in the mistake bound of Lemma~\ref{l:ub_r}.
\end{proof}
\begin{proof}[Corollary \ref{c:ub-tree-r-poly}]
Recall that the resistance between two nodes $i$ and $j$ of any tree is simply
the sum of the inverse weights over all edges on the path connecting
the two nodes. Since $T$ is polynomially connected, we know
that the ratio of any pair of edge weights is polynomial in $n$.
This implies that $R^W_L \Phi^W_L(\by)$ is polynomial in $n$, too.
We apply Theorem~\ref{t:robust} 
to bound $\Phi_L(\by)$
in the mistake bound of Lemma~\ref{l:ub-L-to-T} with $E' = \emptyset$. 
This concludes the proof.
\end{proof}
\begin{lemma}\label{exp_phi_L}
If \wta\ is run on a line graph $L$ obtained by linearizing a random spanning tree $T$ 
of a labeled and weighted graph $(G,\by)$, then we have
\[
	\bE\bigl[\Phi_L(\by)\bigr] \bigoheq \Phi_G^*(\by)~.
\]
\end{lemma}
\begin{proof}
Recall that Theorem~\ref{t:robust} holds for any spanning tree $T$ of $G$.
Thus it suffices to apply part~2 of Theorem~\ref{t:robust} and use $\bE\bigl[ \min X \bigl] \leq \min \bE[X]$~.
\end{proof}
\begin{proof}[Theorem \ref{th:graph_r}]
We apply Lemma~\ref{exp_phi_L} to bound $\bE\bigl[\Phi_L(\by)\bigr]$ in Lemma~\ref{l:graph_lemma}.
\end{proof}
\begin{proof}[Corollary \ref{c:graph_pol_r}]
Since the ratio of the weights of any pair of edges in $G$ is 
polynomial in $n$, the spanning tree $T$ must be polynomially connected.
Thus we can use Corollary~\ref{c:ub-tree-r-poly}, and bound $\bE\bigl[\Phi_L(\by)\bigr]$ via 
Lemma~\ref{exp_phi_L}.
\end{proof}

\section*{Appendix B}\label{app:tables}
This appendix summarizes all our experimental results. For each combination of dataset, algorithm, and
train/test split, we provide macro-averaged error rates and F-measures on the test set.
The algorithms are \wta, \nwwta, and \gpa\ (all combined with various spanning trees), \omv,
\labprop, and \wta\ run with committees of random spanning trees.
WEBSPAM was too large a dataset to perform as thorough an investigation.
Hence we only report test error results on the four algorithms \wta, \omv, \labprop, and \wta\ with
a committee of 7 (nonweighted) random spanning trees.

\begin{table}[h]
\begin{center}
\begin{small}
\begin{tabular}{l|c c|c c|c c|c c}
\ \ \ \ \ Train/test split&\multicolumn{2}{|c|}{5\%}&\multicolumn{2}{|c|}{10\%}&\multicolumn{2}{|c|}{25\%}&\multicolumn{2}{|c}{50\%}\\
Predictors&Error&F&Error&F&Error&F&Error&F\\
\hline
\wta+\rst&25.54&0.81&22.67&0.84&19.06&0.86&16.57&0.88\\
\wta+\nwrst&25.81&0.81&22.70&0.83&19.24&0.86&17.00&0.87\\
\wta+\mst&21.09&0.84&17.94&0.87&13.93&0.90&11.40&0.91\\
\wta+\spst&25.47&0.81&22.65&0.83&19.31&0.86&17.24&0.87\\
\wta+\dfst&26.02&0.81&22.34&0.84&17.73&0.87&14.89&0.89\\
\hline
\nwwta+\rst&25.28&0.81&22.45&0.84&19.12&0.86&17.16&0.87\\
\nwwta+\nwrst&25.97&0.81&23.14&0.83&19.54&0.86&17.84&0.87\\
\nwwta+\mst&21.18&0.84&18.17&0.87&14.51&0.89&12.44&0.91\\
\nwwta+\spst&25.49&0.81&22.81&0.83&19.64&0.86&17.55&0.87\\
\nwwta+\dfst&26.08&0.81&22.82&0.83&17.93&0.87&15.64&0.88\\
\hline
\gpa+\rst&32.75&0.75&29.85&0.78&27.67&0.80&24.44&0.82\\
\gpa+\nwrst&34.27&0.74&30.36&0.78&28.90&0.79&25.99&0.81\\
\gpa+\mst&27.98&0.79&24.89&0.82&21.80&0.84&20.27&0.85\\
\gpa+\spst&27.18&0.79&25.13&0.82&22.20&0.84&20.27&0.85\\
\gpa+\dfst&47.11&0.61&45.65&0.64&43.08&0.66&38.20&0.71\\
\hline
7*\wta+\rst&17.40&0.87&14.85&0.90&12.15&0.91&10.39&0.92\\
7*\wta+\nwrst&17.81&0.87&15.15&0.89&12.51&0.91&10.92&0.92\\
\hline
11*\wta+\rst&16.40&0.88&13.86&0.90&11.38&0.92&9.71&0.93\\
11*\wta+\nwrst&16.78&0.88&14.22&0.90&11.73&0.92&10.20&0.93\\
\hline
17*\wta+\rst&15.78&0.89&13.23&0.91&10.85&0.92&9.22&0.94\\
17*\wta+\nwrst&16.07&0.89&13.55&0.90&11.18&0.92&9.65&0.93\\
\hline
\omv&31.82&0.76&22.27&0.84&11.82&0.91&8.76&0.93\\
\hline
\hline
\labprop&16.33&0.89&13.00&0.91&10.00&0.93&8.77&0.94

\end{tabular}
\end{small}
\end{center}
\caption{
\label{t:rcv1-k10}RCV1-10 - Average error rate and F-measure on 4 classes.
}
\end{table}

\begin{table}[h]
\begin{center}
\begin{small}
\begin{tabular}{l|c c|c c|c c|c c}
\ \ \ \ \ Train/test split&\multicolumn{2}{|c|}{5\%}&\multicolumn{2}{|c|}{10\%}&\multicolumn{2}{|c|}{25\%}&\multicolumn{2}{|c}{50\%}\\
Predictors&Error&F&Error&F&Error&F&Error&F\\
\hline
\wta+\rst&32.03&0.77&29.36&0.79&26.09&0.81&23.25&0.83\\
\wta+\nwrst&32.05&0.77&29.89&0.78&26.65&0.80&23.82&0.83\\
\wta+\mst&20.45&0.85&17.36&0.87&13.91&0.90&11.19&0.92\\
\wta+\spst&29.26&0.79&27.06&0.80&24.96&0.82&23.17&0.83\\
\wta+\dfst&32.03&0.77&28.89&0.79&24.18&0.82&20.57&0.85\\
\hline
\nwwta+\rst&31.72&0.77&29.46&0.78&26.20&0.81&24.04&0.82\\
\nwwta+\nwrst&32.52&0.76&29.95&0.78&26.88&0.80&24.84&0.82\\
\nwwta+\mst&20.54&0.85&17.68&0.87&14.37&0.89&12.25&0.91\\
\nwwta+\spst&29.28&0.79&27.13&0.80&25.16&0.82&23.72&0.83\\
\nwwta+\dfst&32.05&0.77&28.81&0.79&24.14&0.82&21.28&0.84\\
\hline
\gpa+\rst&36.47&0.73&35.33&0.74&33.81&0.75&32.32&0.76\\
\gpa+\nwrst&38.26&0.72&35.91&0.73&35.20&0.74&32.73&0.76\\
\gpa+\mst&26.65&0.81&24.30&0.82&20.29&0.85&18.75&0.86\\
\gpa+\spst&32.43&0.74&28.00&0.78&26.61&0.79&25.77&0.80\\
\gpa+\dfst&48.35&0.61&47.85&0.61&44.78&0.65&41.12&0.68\\
\hline
7*\wta+\rst&23.30&0.84&20.55&0.86&16.87&0.88&14.34&0.90\\
7*\wta+\nwrst&23.64&0.84&20.77&0.86&17.27&0.88&14.81&0.90\\
\hline
11*\wta+\rst&22.06&0.85&19.39&0.87&15.63&0.89&13.20&0.91\\
11*\wta+\nwrst&22.29&0.85&19.54&0.87&16.09&0.89&13.61&0.91\\
\hline
17*\wta+\rst&21.33&0.86&18.62&0.88&14.91&0.90&12.39&0.92\\
17*\wta+\nwrst&21.49&0.86&18.86&0.87&15.29&0.89&12.78&0.91\\
\hline
\omv&12.48&0.91&10.50&0.93&9.49&0.93&8.96&0.94\\
\hline
\hline
\labprop&24.39&0.85&20.78&0.87&14.45&0.91&10.73&0.93
\end{tabular}
\end{small}
\end{center}
\caption{\label{t:rcv1k100}RCV1-100 - Average error rate and F-measure on 4 classes.
}
\end{table}

\begin{table}[h]
\begin{center}
\begin{small}
\begin{tabular}{l|c c|c c|c c|c c}
\ \ \ \ \ Train/test split&\multicolumn{2}{|c|}{5\%}&\multicolumn{2}{|c|}{10\%}&\multicolumn{2}{|c|}{25\%}&\multicolumn{2}{|c}{50\%}\\
Predictors&Error&F&Error&F&Error&F&Error&F\\
\hline
\wta+\rst&5.32&0.97&4.28&0.98&3.08&0.98&2.36&0.99\\
\wta+\nwrst&5.65&0.97&4.51&0.97&3.29&0.98&2.56&0.98\\
\wta+\mst&1.98&0.99&1.61&0.99&1.24&0.99&0.94&0.99\\
\wta+\spst&6.25&0.97&4.72&0.97&3.37&0.98&2.60&0.99\\
\wta+\dfst&6.43&0.96&4.60&0.97&2.92&0.98&2.04&0.99\\
\hline
\nwwta+\rst&5.31&0.97&4.25&0.98&3.19&0.98&2.70&0.99\\
\nwwta+\nwrst&5.95&0.97&4.65&0.97&3.45&0.98&2.92&0.98\\
\nwwta+\mst&1.99&0.99&1.59&0.99&1.29&0.99&1.06&0.99\\
\nwwta+\spst&6.30&0.96&4.83&0.97&3.50&0.98&2.84&0.98\\
\nwwta+\dfst&6.49&0.96&4.59&0.97&3.09&0.98&2.35&0.99\\
\hline
\gpa+\rst&12.64&0.93&8.53&0.95&6.65&0.96&5.65&0.97\\
\gpa+\nwrst&12.53&0.93&9.05&0.95&6.90&0.96&5.19&0.97\\
\gpa+\mst&2.58&0.99&3.18&0.98&2.28&0.99&1.48&0.99\\
\gpa+\spst&7.64&0.96&6.26&0.96&4.13&0.98&3.55&0.98\\
\gpa+\dfst&42.77&0.70&39.39&0.73&32.38&0.79&20.53&0.87\\
\hline
7*\wta+\rst&2.09&0.99&1.56&0.99&1.14&0.99&0.90&0.99\\
7*\wta+\nwrst&2.35&0.99&1.75&0.99&1.26&0.99&1.02&0.99\\
\hline
11*\wta+\rst&1.84&0.99&1.35&0.99&1.01&0.99&0.82&1.00\\
11*\wta+\nwrst&2.05&0.99&1.53&0.99&1.14&0.99&0.91&0.99\\
\hline
17*\wta+\rst&1.65&0.99&1.23&0.99&0.95&0.99&0.77&1.00\\
17*\wta+\nwrst&1.87&0.99&1.39&0.99&1.06&0.99&0.85&1.00\\
\hline
\omv&24.84&0.85&12.28&0.93&2.13&0.99&0.75&1.00\\
\hline
\hline
\labprop&2.14&0.99&1.16&0.99&0.85&0.99&0.73&1.00
\end{tabular}
\end{small}
\end{center}
\caption{\label{t:uspsk10}USPS-10 - Average error rate and F-measure  on 10 classes.
}
\end{table}

\begin{table}[h]
\begin{center}
\begin{small}
\begin{tabular}{l|c c|c c|c c|c c}
\ \ \ \ \ Train/test split&\multicolumn{2}{|c|}{5\%}&\multicolumn{2}{|c|}{10\%}&\multicolumn{2}{|c|}{25\%}&\multicolumn{2}{|c}{50\%}\\
Predictors&Error&F&Error&F&Error&F&Error&F\\
\hline
\wta+\rst&9.62&0.95&8.29&0.95&6.55&0.96&5.36&0.97\\
\wta+\nwrst&10.32&0.94&9.00&0.95&7.17&0.96&5.83&0.97\\
\wta+\mst&1.90&0.99&1.49&0.99&1.22&0.99&0.94&0.99\\
\wta+\spst&8.68&0.95&7.27&0.96&5.78&0.97&4.88&0.97\\
\wta+\dfst&10.36&0.94&8.13&0.96&5.62&0.97&4.21&0.98\\
\hline
\nwwta+\rst&9.71&0.95&8.38&0.95&6.78&0.96&5.89&0.97\\
\nwwta+\nwrst&10.39&0.94&9.08&0.95&7.46&0.96&6.45&0.96\\
\nwwta+\mst&1.91&0.99&1.60&0.99&1.23&0.99&1.09&0.99\\
\nwwta+\spst&8.76&0.95&7.46&0.96&5.94&0.97&5.28&0.97\\
\nwwta+\dfst&10.46&0.94&8.30&0.95&6.00&0.97&4.65&0.97\\
\hline
\gpa+\rst&14.81&0.91&13.38&0.92&11.94&0.93&9.81&0.94\\
\gpa+\nwrst&17.34&0.90&13.68&0.92&11.39&0.94&11.46&0.94\\
\gpa+\mst&3.57&0.98&2.26&0.99&1.77&0.99&1.39&0.99\\
\gpa+\spst&8.42&0.95&7.94&0.95&7.20&0.96&5.71&0.97\\
\gpa+\dfst&46.09&0.67&42.59&0.71&37.66&0.75&28.45&0.82\\
\hline
7*\wta+\rst&5.28&0.97&4.24&0.98&3.05&0.98&2.37&0.99\\
7*\wta+\nwrst&5.82&0.97&4.73&0.97&3.48&0.98&2.69&0.98\\
\hline
11*\wta+\rst&5.07&0.97&3.96&0.98&2.76&0.99&2.11&0.99\\
11*\wta+\nwrst&5.55&0.97&4.38&0.98&3.14&0.98&2.40&0.99\\
\hline
17*\wta+\rst&5.17&0.97&3.96&0.98&2.72&0.99&2.05&0.99\\
17*\wta+\nwrst&7.60&0.96&6.38&0.97&4.68&0.97&3.32&0.98\\
\hline
\omv&2.17&0.99&1.70&0.99&1.53&0.99&1.45&0.99\\
\hline
\hline
\labprop&6.94&0.96&5.19&0.97&2.51&0.99&1.79&0.99
\end{tabular}
\end{small}
\end{center}
\caption{\label{t:uspsk100}USPS-100 - Average error rate and F-measure on 10 classes.
}
\end{table}

\begin{table}[h]
\begin{center}
\begin{small}
\begin{tabular}{l|c c|c c|c c|c c}
\ \ \ \ \ Train/test split&\multicolumn{2}{|c|}{5\%}&\multicolumn{2}{|c|}{10\%}&\multicolumn{2}{|c|}{25\%}&\multicolumn{2}{|c}{50\%}\\
Predictors&Error&F&Error&F&Error&F&Error&F\\
\hline
\wta+\rst&21.73&0.86&21.37&0.86&19.89&0.87&19.09&0.88\\
\wta+\nwrst&21.86&0.86&21.50&0.86&20.03&0.87&19.33&0.88\\
\wta+\mst&21.55&0.86&20.86&0.87&19.35&0.88&18.36&0.88\\
\wta+\spst&21.86&0.86&21.58&0.86&20.38&0.87&19.40&0.88\\
\wta+\dfst&21.78&0.86&21.22&0.86&19.88&0.87&18.60&0.88\\
\hline
\nwwta+\rst&21.83&0.86&21.43&0.86&20.08&0.87&19.64&0.88\\
\nwwta+\nwrst&21.98&0.86&21.55&0.86&20.26&0.87&19.75&0.87\\
\nwwta+\mst&21.55&0.86&20.91&0.87&19.55&0.88&18.89&0.88\\
\nwwta+\spst&21.86&0.86&21.57&0.86&20.50&0.87&19.81&0.87\\
\nwwta+\dfst&21.79&0.86&21.33&0.86&20.00&0.87&19.09&0.88\\
\hline
\gpa+\rst&22.70&0.85&22.75&0.85&22.14&0.86&21.28&0.86\\
\gpa+\nwrst&23.83&0.84&23.28&0.85&22.48&0.85&21.53&0.86\\
\gpa+\mst&21.99&0.86&21.34&0.86&20.77&0.86&20.48&0.87\\
\gpa+\spst&22.33&0.84&21.34&0.86&20.71&0.86&20.74&0.86\\
\gpa+\dfst&39.77&0.72&31.93&0.78&25.70&0.83&24.09&0.84\\
\hline
7*\wta+\rst&16.83&0.90&16.63&0.90&15.78&0.90&15.29&0.90\\
7*\wta+\nwrst&16.85&0.90&16.60&0.90&15.89&0.90&15.41&0.90\\
\hline
11*\wta+\rst&16.28&0.90&16.11&0.90&15.36&0.91&14.92&0.91\\
11*\wta+\nwrst&16.28&0.90&16.08&0.90&15.55&0.90&14.99&0.91\\
\hline
17*\wta+\rst&15.93&0.90&15.78&0.90&15.17&0.91&14.63&0.91\\
17*\wta+\nwrst&15.98&0.90&15.69&0.91&15.23&0.91&14.68&0.91\\
\hline
\omv&42.98&0.70&38.88&0.73&29.85&0.80&22.66&0.85\\
\hline
\hline
\labprop&15.26&0.91&15.21&0.91&14.94&0.91&15.13&0.91
\end{tabular}
\end{small}
\end{center}
\caption{\label{t:krogan}KROGAN - Average error rate and F-measure on 17 classes.
}
\end{table}

\begin{table}[h]
\begin{center}
\begin{small}
\begin{tabular}{l|c c|c c|c c|c c}
\ \ \ \ \ Train/test split&\multicolumn{2}{|c|}{5\%}&\multicolumn{2}{|c|}{10\%}&\multicolumn{2}{|c|}{25\%}&\multicolumn{2}{|c}{50\%}\\
Predictors&Error&F&Error&F&Error&F&Error&F\\
\hline
\wta+\rst&21.68&0.86&21.05&0.87&20.08&0.87&18.99&0.88\\
\wta+\nwrst&21.47&0.87&21.29&0.86&20.18&0.87&19.17&0.88\\
\wta+\mst&21.57&0.86&20.63&0.87&19.61&0.88&18.37&0.88\\
\wta+\spst&21.39&0.87&21.34&0.86&20.52&0.87&19.57&0.88\\
\wta+\dfst&21.88&0.86&21.09&0.87&19.82&0.87&18.83&0.88\\
\hline
\nwwta+\rst&21.50&0.87&21.15&0.87&20.43&0.87&19.95&0.87\\
\nwwta+\nwrst&21.61&0.86&21.26&0.87&20.52&0.87&20.09&0.87\\
\nwwta+\mst&21.53&0.86&20.95&0.87&20.35&0.87&19.81&0.88\\
\nwwta+\spst&21.37&0.87&21.06&0.87&20.55&0.87&20.06&0.87\\
\nwwta+\dfst&21.88&0.86&21.05&0.87&20.50&0.87&19.74&0.88\\
\hline
\gpa+\rst&23.56&0.85&22.27&0.86&21.86&0.86&21.68&0.86\\
\gpa+\nwrst&23.91&0.85&23.11&0.85&22.47&0.86&21.30&0.86\\
\gpa+\mst&23.32&0.85&21.60&0.86&21.77&0.86&21.67&0.86\\
\gpa+\spst&22.55&0.85&21.89&0.85&21.64&0.85&21.70&0.85\\
\gpa+\dfst&41.69&0.71&30.82&0.79&26.75&0.82&23.56&0.84\\
\hline
7*\wta+\rst&16.39&0.90&16.09&0.90&15.77&0.91&15.29&0.91\\
7*\wta+\nwrst&16.35&0.90&16.10&0.90&15.77&0.90&15.47&0.91\\
\hline
11*\wta+\rst&15.89&0.91&15.61&0.91&15.32&0.91&14.84&0.91\\
11*\wta+\nwrst&15.82&0.91&15.57&0.91&15.34&0.91&14.98&0.91\\
\hline
17*\wta+\rst&15.54&0.91&15.31&0.91&14.97&0.91&14.55&0.91\\
17*\wta+\nwrst&15.45&0.91&15.29&0.91&15.05&0.91&14.66&0.91\\
\hline
\omv&44.74&0.68&40.75&0.72&32.97&0.78&25.28&0.84\\
\hline
\hline
\labprop&14.93&0.91&14.98&0.91&15.23&0.91&15.31&0.90
\end{tabular}
\end{small}
\end{center}
\caption{\label{t:combined}COMBINED - Average error rate and F-measure on 17 classes.
}
\end{table}

\begin{table}[h]
\begin{small}
\begin{center}
\begin{tabular}{ l|c|c }

&\multicolumn{2}{c}{ }\\
Predictors &Error &F\\
\hline
\wta+\nwrst&10.03&0.95\\
3*\wta+\nwrst&6.44&0.97\\
7*\wta+\nwrst&5.91&0.97\\
\hline
\omv&44.1&0.71\\
\hline
\hline
\labprop&12.84&0.93
\end{tabular}
\end{center}
\end{small}
\caption{\label{t:spamresults}WEBSPAM - Test set error rate and F-measure. $\wta$ operates only on \nwrst.
}
\end{table}


\end{document}